\documentclass{article}

\usepackage[final]{sections-neurips/neurips_2024}

\usepackage[utf8]{inputenc} 
\usepackage[T1]{fontenc}    
\usepackage{url}            
\usepackage{booktabs}       
\usepackage{amsfonts}       
\usepackage{nicefrac}       
\usepackage{microtype}      
\usepackage{mathtools}
\usepackage{bm}
\usepackage{float}
\usepackage{wrapfig}
\usepackage{amsthm}
\usepackage{amsmath}
\usepackage{amssymb}
\usepackage{booktabs}
\usepackage{graphicx}
\usepackage{graphbox}
\usepackage{notes2bib}
\usepackage{multirow}
\usepackage{algorithm, algorithmic}
\usepackage{threeparttable}
\usepackage{enumitem}
\usepackage{setspace}

\usepackage[colorlinks,linkcolor=blue]{hyperref}

\usepackage[capitalize,noabbrev]{cleveref}
\crefname{section}{Section}{Sections}
\crefname{theorem}{Theorem}{Theorems}
\crefname{lemma}{Lemma}{Lemmas}
\crefname{equation}{Equation}{Equations}
\crefname{proposition}{Proposition}{Propositions}
\crefname{claim}{Claim}{Claims}
\crefname{appendix}{Appendix}{Appendices}
\crefname{algorithm}{Algorithm}{Algorithms}
\crefname{figure}{Figure}{Figures}
\crefname{table}{Table}{Tables}
\crefname{remark}{Remark}{Remarks}
\crefname{definition}{Def.}{Definitions}
\crefname{corollary}{Corollary}{Corollaries}

\usepackage{bm}
\usepackage{float}
\usepackage{amsthm}
\usepackage{amsmath}
\usepackage{amssymb}
\usepackage{booktabs}
\usepackage{graphicx}
\usepackage{graphbox}
\usepackage{notes2bib}
\usepackage{subcaption}
\usepackage{multirow}
\usepackage{algorithm, algorithmic}

\usepackage[colorlinks,linkcolor=blue]{hyperref}
\usepackage[table,RGB,dvipsnames]{xcolor}
\definecolor{cite_color}{HTML}{114083}
\definecolor{link_color}{RGB}{0,102,102}
\definecolor{link_color}{RGB}{153, 0,0}  
\definecolor{url_color}{RGB}{153, 102,  0}
\definecolor{emp_color}{RGB}{0,0,255}
\hypersetup{
 colorlinks,
 citecolor=cite_color,
 linkcolor=link_color,
 urlcolor=url_color}
\graphicspath{{./figs/}}

\DeclarePairedDelimiterX{\infdivx}[2]{(}{)}{%
  #1\;\delimsize\|\;#2%
  }

\newcommand{\infdiv}{\operatorname{KL}\infdivx}

\def \x{\mathbf{x}}
\def \y{\mathbf{y}}

\def \z{\mathbf{z}}
\def \xib{\boldsymbol{\xi}}
\def \iu{\mathrm{i}}

\def \ED{\mathrm{ED}}
\def \KL{\mathrm{KL}}

\def \data{\mathrm{data}}
\def \ebm{\mathrm{ebm}}
\def \e{\mathbf{e}}

\newtheorem{innercustomgeneric}{\customgenericname}
\providecommand{\customgenericname}{}
\newcommand{\newcustomtheorem}[2]{%
  \newenvironment{#1}[1]
  {%
   \renewcommand\customgenericname{#2}%
   \renewcommand\theinnercustomgeneric{##1}%
   \innercustomgeneric
  }
  {\endinnercustomgeneric}
}

\newcustomtheorem{customthm}{Theorem}

\DeclareMathOperator*{\argmin}{argmin}

\newtheorem{definition}{Definition}

\usepackage{thmtools}
\usepackage{thm-restate}

\usepackage{tikz}
\usetikzlibrary{decorations.pathreplacing,calc}

\usepackage{xspace}
\newcommand\EDB{\text{ED-Bern}\xspace}

\newcommand\EDG{\text{ED-Grid}\xspace}

\newcommand\num{\mathrm{num}}
\newcommand\cat{\mathrm{cat}}

\usepackage{etoc}
\etocdepthtag.toc{mtchapter}
\etocsettagdepth{mtchapter}{subsection}
\etocsettagdepth{mtappendix}{none}

\title{Energy-Based Modelling for Discrete and Mixed Data via Heat Equations on Structured Spaces}

\author{
Tobias Schröder \thanks{Equal contributions}\hspace{1.5mm}\thanks{Correspondence to: \texttt{t.schroeder21@imperial.ac.uk} \, and 
\,\texttt{z.ou22@imperial.ac.uk}}\\
Imperial College London\\
\And Zijing Ou$^\ast$\thanks{Code: \url{https://github.com/J-zin/discrete-energy-discrepancy}}\\
Imperial College London\AND
Yingzhen Li \\
Imperial College London\And 
Andrew Duncan\\
Imperial College London\\
The Alan Turing Institute
}

\begin{document}

\maketitle

\begin{abstract}
  Energy-based models (EBMs) offer a flexible framework for probabilistic modelling across various data domains. However, training EBMs on data in discrete or mixed state spaces poses significant challenges due to the lack of robust and fast sampling methods. In this work, we propose to train discrete EBMs with Energy Discrepancy, a loss function which only requires the evaluation of the energy function at data points and their perturbed counterparts, thus eliminating the need for Markov chain Monte Carlo. We introduce perturbations of the data distribution by simulating a diffusion process on the discrete state space endowed with a graph structure. This allows us to inform the choice of perturbation from the structure of the modelled discrete variable, while the continuous time parameter enables fine-grained control of the perturbation. Empirically, we demonstrate the efficacy of the proposed approaches in a wide range of applications, including the estimation of discrete densities with non-binary vocabulary and binary image modelling. Finally, we train EBMs on tabular data sets with applications in synthetic data generation and calibrated classification. 
\end{abstract}

\section{Introduction}
Discrete structures are intrinsic to most types of data such as text, graphs, and images. Estimating the data generating distribution $p_\data$ of discrete data sets with a probabilistic model can contribute greatly to downstream inference and generation tasks, and plays a key role in synthetic data generation of tabular, textual or network data \citep{raghunathan2021synthetic}. Energy-based models (EBMs) are probabilistic generative models of the form $p_\ebm \propto \exp(-U)$, where the flexible choice of the energy function $U$ allows great control in the modelling of different data structures. However, energy-based models are, by definition, unnormalised models and notoriously difficult to train due to the intractability of their normalisation, especially in discrete spaces.

Energy-based models are typically trained with the contrastive divergence (CD) algorithm \citep{hinton2002training} which performs approximate maximum likelihood estimation by approximating the gradient of the log-likelihood with Markov Chain Monte Carlo (MCMC) techniques. This method motivated rich research results on sampling from discrete distributions to enable fast and accurate estimation of energy-based models \citep{zanella2020informed,grathwohl2021oops,zhang2022langevin,sun2022path,sun2022optimalscaling,sun2023discrete}. However, the training of energy-based models with CD remains challenging as it relies on sufficiently fast mixing of Markov chains. Since accurate sampling from the EBM typically cannot be achieved, contrastive divergence lacks theoretical guarantees \citep{carreira2005contrastive} and leads to biased estimates of the energy landscape \citep{nijkamp2019learning}. For mixed data types, energy-based models have only been applied to combinations of numerical features with a single categorical label \citep{grathwohl2019your}.

The recently introduced Energy Discrepancy (ED) \citep{schroeder2023energy} is a new type of contrastive loss functional that, by definition, depends on neither gradients nor MCMC methods. Instead, the definition of ED only requires the evaluation of the energy function on positive and contrasting, negative samples which are generated by perturbing the data distribution. However, the work in \citet{schroeder2023energy} is currently limited to Gaussian perturbations on continuous spaces and does not explores strategies to choose perturbations on discrete spaces, especially when these discrete spaces exhibit some additional structure.

In this work, we propose a framework to train energy-based models with energy discrepancy on discrete data, making the following contributions: 1) We explore a method to define discrete diffusion processes on structured discrete spaces through a heat equation on  the underlying graph and investigate the effect of geometry and time parameter on the diffusion. 2) Based on the discrete diffusion process, we extend energy discrepancy to discrete spaces in a systematic way, thus introducing a MCMC-free method for the training of energy-based models that requires little tuning.
3) We extend our methodology to mixed state spaces and establish to the best of our knowledge the first robust training method of energy-based models on tabular data sets. We demonstrate promising performance on downstream tasks like synthetic data generation and calibrated prediction, thus unlocking a new tool for generative modelling on tabular data.
\section{Preliminaries}
Energy-based models (EBMs) are a parametric family of distributions $p_\theta$ defined as
\begin{align}
    p_\theta(\x) = \frac{\exp(-U_\theta(\x))}{Z_\theta}, \quad Z_\theta = \sum\nolimits_{\x \in \mathcal{X}} \exp(-U_\theta (\x)),
\end{align}
where $U_\theta$ is the energy function parameterised by $\theta$ and $Z_\theta$ denotes the normalisation constant.
Given a set of {\it i.i.d.} samples $\{\x^i\}_{i=1}^N$ from an unknown data distribution $p_\data (\x)$ we aim to learn an approximation $p_\theta(\x)$ of $p_\data(\x)$. The {\it de facto} standard approach for finding such $\theta$ is to minimise the negative log-likelihood of $p_\theta$ under the data distribution via gradient decent
\begin{align} \label{mle-objective}
    -\nabla_\theta \mathbb{E}_{p_\data (\x)} [\log p_\theta (\x)] = \mathbb{E}_{\x \sim p_\data}[\nabla_\theta U_\theta (\x)] - \mathbb{E}_{{\x}_- \sim p_\theta}[\nabla_\theta U_\theta ({\x}_-)].
\end{align}
The intuition behind this update is to decrease the energy of positive data samples $\x \sim p_\data(\x)$ and to increase the energy of negative samples ${\x}_- \sim p_\theta(\x)$. However, the exact computation of gradient in \eqref{mle-objective} is known to be NP-hard in general \citep{jerrum1993polynomial} and quickly becomes prohibitive even on relatively simple data sets.  Consequently, existing approaches resort to sampling from the model $p_\theta$ to approximate the gradient of log-likelihood via Monte Carlo estimation. In discrete settings, the most popular sampling methods include the locally informed sampler \citep{zanella2020informed}, Gibbs with gradients (GwG) \citep{grathwohl2021oops}, discrete Langevin \citep{zhang2022langevin}, and generative flow networks (GFlowNet) \citep{zhang2022generative}. Despite their established success in discrete energy-based modelling, these methods necessitate a trade-off that hampers scalability: running the sampler for an extended duration rapidly increases the cost of maximum likelihood training, while shorter sampler runs yield inaccurate approximations of the likelihood gradient and introduce biases into the learned energy.

Energy Discrepancy \citep{schroeder2023energy} is a recently proposed method to train energy-based models without the need for an extensive sampling process. Instead, it constructs negative samples by perturbing the data, thus bypassing the sampling step while still yielding a valid training objective. To elucidate, the energy discrepancy is formally defined as follows:

\begin{definition} [Energy Discrepancy]
    Let $p_{\mathrm{data}}$ be a positive density on a measure space ($\mathcal{X}$, $\mathrm d\x$)\footnotemark and let $q(\y|\x)$ be a conditional probability density. Define the \emph{contrastive potential} induced by $q$ as\footnotemark
    \begin{align} \label{definition-contrastive-potential}
        U_q (\y) := - \log \sum_{\x'\in\mathcal X} q(\y | \x') \exp(-U(\x'))
    \end{align}
    The \emph{energy discrepancy} between $p_{\mathrm{data}}$ and $U$ induced by $q$ is defined as
    \addtocounter{footnote}{-2}
    \stepcounter{footnote}\footnotetext{On discrete spaces $\mathrm d\x$ is assumed to be a counting measure. On continuous spaces $\mathcal X$, the appearing sums and expectations turn into integrals with respect to the Lebesgue measure}
    \stepcounter{footnote}\footnotetext{With a slight abuse of notations, we represent the contrastive potential induced by distribution $q$ as $U_q$ and denote the energy function as $U$ with or without the subscript $\theta$.}
    \begin{equation} \label{energy-discrepancy-loss}
        \ED_q (p_{\mathrm{data}}, U) := \mathbb{E}_{p_{\mathrm{data}}(\x)} [U(\x)] - \mathbb{E}_{p_{\mathrm{data}}(\x)}\mathbb{E}_{q(\y | \x)} [U_q (\y)]. 
    \end{equation}
\end{definition}

We will refer to $q$ as the \emph{perturbation}. The validity of this loss functional was proven in \citet{schroeder2023energy} in large generality: In particular, it is sufficient for $U^\ast = \argmin\ED_q (p_{\mathrm{data}}, U) \Leftrightarrow \exp(-U^\ast) \propto p_\data$ that any two points $\x, \y \in \mathcal X$ are $q$-equivalent, i.e. there exists a chain of states $(\z^i)_{i=1}^T \in \mathcal X$ with $\z^1 = \x, \z^T = \y$ such that $q(\z^{i+1} \vert \z^i) >0$ for all $i = 1, \dots, T-1$.

Energy discrepancy can also be understood from seeing it as a type of Kullback-Leibler divergence. Specifically, the loss function defined in \eqref{energy-discrepancy-loss} is equivalent to the expected Kullback-Leibler divergence
\begin{align}
    \argmin_U \ED_q (p_\data, U) \Leftrightarrow \argmin_U \mathbb E_{p_\data(\x)}\mathbb E_{q(\y\vert\x)}\left[\KL (p_\data(\cdot \vert\y), p_\ebm(\cdot\vert\y))\right]
\end{align}
where $p_\bullet(\x \vert \y)\propto p_\bullet(\cdot)q(\y\vert \cdot)$. This relates energy discrepancy to diffusion recovery likelihood objectives \citep{gao2020learning} and Kullback-Leibler contractions \citep{Lyu2011KLContractions}.
Energy discrepancy has demonstrated notable effectiveness in training EBMs in continuous spaces \citep{schroeder2023energy}. In the next section, we show how the loss can be defined in discrete spaces.

\section{Energy Discrepancies for Discrete Data}\label{sec:DiscreteED}
For this work we will first consider a state space for the data distribution that can be written as the product of $d$ discrete variables with $S_k$ classes each, i.e. $\mathcal X = \bigotimes_{k = 1}^{d} \{1, \dots, S_k\}$. Examples for spaces of this type are the categorical entries of a data table for which $d$ denotes the number of features, or binary image data sets for which we typically write $\mathcal X = \{0, 1\}^d$. To define energy discrepancy on such spaces we need to specify a perturbation process under the following considerations: 1) The negative samples obtained through $q$ are informative for training the EBM when only finite amounts of data are available. 2) The contrastive potential $U_q(\y)$ has a numerically tractable approximation.

Let us consider one component $\mathcal X = \{1, \dots, S\}$, only. Inspired from previous works on diffusion modelling for discrete data \citep{campbell2022continuous, sun2022score, lou2023discrete, campbell2024generative} we model the perturbation as a continuous time Markov chain (CTMC) with transition probability
\begin{equation*}
    q_t(y = b\vert x = a) = \exp\left(tR\right)_{ba}\,, \quad a, b \in \{1, 2, \dots, S\}
\end{equation*}
where $R\in \mathbb R^{S\times S}$ is the so-called rate matrix which satisfies $R_{bb} = -\sum_{a\neq b}^S R_{ba}$ and $\exp(tR)$ is the matrix exponential. 
For a given rate matrix $R$, this approach then leaves us with a single tunable hyperparameter $t$ characterising the magnitude of perturbation applied. We first analyse how the choice of rate matrix and time parameter affect the statistical properties of the energy discrepancy loss. In fact, under weak conditions, the energy discrepancy loss converges to maximum likelihood estimation for $t\to \infty$, thus achieving the same loss function implemented by contrastive divergence:
\begin{restatable}[]{theorem}{restatheoremone}
\label{TheoremMLElimit}
Let $q_t(\cdot\vert x)$ be a Markov transition density defined by the rate matrix $R$ with eigenvalues $0 =  \lambda_1(R) \geq \lambda_2(R) \geq \dots \geq\lambda_S(R)$ and uniform stationary distribution. Then, there exists a constant $z_t$ \emph{independent} of $\theta$ such that energy-discrepancy converges to the maximum-likelihood loss
\begin{equation*}
    \left\lvert \ED_{q_t}(p_\data, U_\theta) -\mathcal L_{\mathrm{MLE}}(\theta) - z_t\right\rvert\leq \sqrt{S}\exp(-\vert\lambda_2(R)\vert t) \infdiv{p_\data}{p_\theta}\,.
\end{equation*}
with the loss of maximum-likelihood estimation $\mathcal L_{\mathrm{MLE}}(\theta) := -\mathbb E_{p_\data(\x)}\big[\log p_\theta(\x)\big]$.
\end{restatable}
Here, $z_t$ is a constant independent of $\theta$, so the optimisation landscapes of energy discrepancy estimation and maximum likelihood estimation in $\theta$ align at an exponential rate, except for a shift by $z_t$ which does not affect the optimisation. This result improves the linear convergence rate in \citet{schroeder2023energy} and relates it to the \emph{spectral gap} $\vert \lambda_2(R)\vert$ of the rate matrix. Such a result is meaningful as the maximum-likelihood estimator is generally statistically preferable with better sample efficiency, and \cref{TheoremMLElimit} suggests that energy discrepancy estimation can approximate maximum likelihood estimation without resorting to MCMC like in classical EBM training methods. The proof is given in \cref{pf:TheoremMLELimit}.
\subsection{Heat Equation in Structured Discrete Spaces}
\begin{figure}
    \centering
    \includegraphics[width = \textwidth]{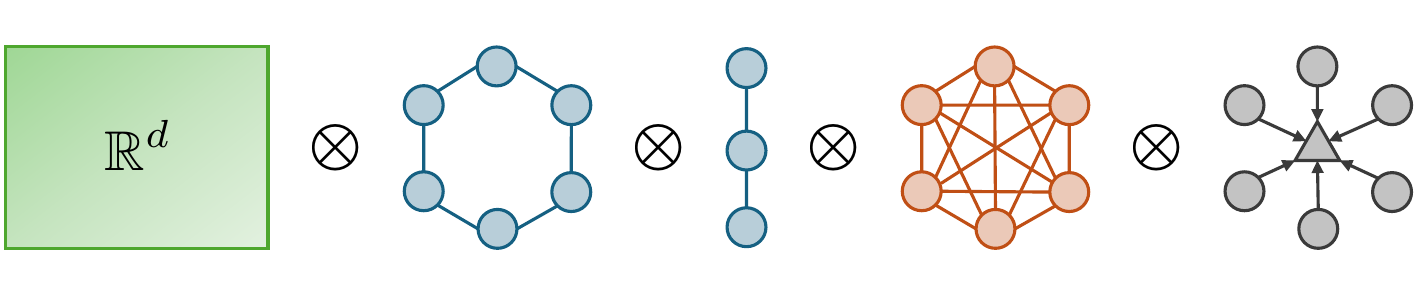}\label{figure_state_space}
    \caption{Visualisation of a typical state space of a tabular dataset: Numerical entries taking values in $\mathbb R^d$, cyclical categorical entries (e.g. season), ordinal categorical entries (e.g. age), unstructured categorical entries, and variables with an absorbing state associated with masking the entry.}
    \label{fig:enter-label}
\end{figure}

In principle, \cref{TheoremMLElimit} establishes that energy discrepancy converges to the loss of maximum likelihood estimation in the limit $t\to\infty$ for any choice of rate matrix with spectral gap. 
In practice, however, large perturbations of data can produce high-variance parameter gradients and provide little training signal. Instead, it is desirable to construct perturbations that allow a fine-grained trade-off between the statistical properties of the loss function and the variance of the gradients. For this reason, we investigate the perturbation for small $t$ which, as we will see, can be informed by the assumed \emph{graph structure} of the underlying discrete space.

The infinitesimal perturbations of the CTMC are characterised by the rate matrix. Notice that for small $t$, the Euler discretisation of the heat equation yields
\begin{equation*}
    q_{t}(y = b\vert x= a) \approx \delta(b, a) + t R_{ba}\,.
\end{equation*}
with $\delta(b, a) = 1 \Leftrightarrow a = b$ and zero otherwise. To model the relationship between values $a, b \in \{1, \dots, S\}$ we endow the space with a graph structure with adjacency matrix $A$ and out degree matrix $D_{\mathrm{out}}$ and model the rate matrix as the graph Laplacian $R := (A - D_{\mathrm{out}})$. By definition, the rows of the graph Laplacian matrix sum to zero $\sum_{a = 1}^S R_{ba} = 0$. The smallest possible perturbation is then characterised as the transition to an adjacent neighbour. The characterisation of the CTMC in terms of the graph Laplacian is implicitly assumed in previous work. \citet{campbell2022continuous} describe a diffusion via a uniform perturbation which corresponds to a fully connected graph and \citet{lou2023discrete} describe the rate matrix associated to a star graph with absorbing (masking) state:
\begin{align*}
    R^{\mathrm{unif}} = \boldsymbol{1} \boldsymbol{1}^T/S - \mathrm{id} \,,\quad\quad R^{\mathrm{mask}}_{ba} = \delta(M, a) - \mathrm{id}_{ba}\,.
\end{align*}
For a visualisation, see \cref{figure_state_space}. In addition to these fairly unstructured rate matrices we model state spaces with a cyclical or ordinal structure:
\begin{equation*}
    R^{\mathrm{cyc}} = \scalebox{0.8}{$\begin{pmatrix}
        -2 & 1 & 0 & \cdots & 0 & 1\\
        1 & -2 & 1 &0&\cdots& 0\\
        0 &-1& -2 & 1&0&\vdots\\
        \vdots & \ddots &\ddots&\ddots& \ddots& \ddots \\
        0 & &0&1& -2 & 1\\
        1 & 0 &\cdots&0&1&-2
    \end{pmatrix}
    $}
    \quad\quad
    R^{\mathrm{ord}} = \scalebox{0.8}{$\begin{pmatrix}
        -1 & 1 & 0 & \cdots & 0 & 0\\
        1 &-2 & 1 &0&\cdots& 0\\
        0 & 1&-2& 1&0&\vdots\\
        \vdots & \ddots &\ddots&\ddots& \ddots& \ddots \\
        0 &  &0&1&-2&1\\
        0 & 0 &\cdots&0& 1&- 1
    \end{pmatrix} $}\,.
\end{equation*}

We restrict ourselves to uniform, cyclical, and ordinal structures as these structures can typically be trivially inferred from the type of data modelled. For example, periodically changing quantities (e.g. seasons) would display a cyclical structure and ordered information like age an ordinal structure. It is possible, however, to extend our framework to arbitrary graphical structures of the state space as long as eigenvalue decompositions of the graph Laplacian are feasible.

Since the Gaussian perturbation on Euclidean space used in \citet{schroeder2023energy} is also the solution of a heat equation, these choices allow us to model the perturbation on a vector of mixed entries including numerical values, unstructured categorical values, and structured categorical values with a single differential equation
\begin{equation} \label{eq:pde-heat-diffuion}
    \partial_t q_t(\cdot \vert \x) = \left(\Delta^{\mathrm{num}} \otimes R^{\mathrm{cyc}} \otimes R^{\mathrm{ord}}\otimes R^{\mathrm{unif}} \otimes R^{\mathrm{abs}}\right) q_t(\cdot \vert \x)\,, \quad q_0(\cdot \vert \x) = \delta(\cdot, \mathbf x)
\end{equation}
where $\Delta^{\mathrm{num}}$ denotes the standard Laplace operator $\sum_{k=1}^d \partial^2/\partial_{\mathbf x^\mathrm{num}}^2$ and the product $\otimes$ denotes that each operator acts on the corresponding component of the state space. 

\section{Estimating the Energy Discrepancy Loss}\label{sec:EstimatingED}
We now discuss how discrete energy discrepancy can be estimated. We will typically assume that each dimension of the data point is perturbed independently, i.e. the perturbation $q(\mathbf y \vert \mathbf x)$ is modelled as the product of component-wise perturbations. On Euclidean data, we resort to the implementation in \citet{schroeder2023energy} and obtain perturbed samples by adding isotropic Gaussian noise to the samples. We are now left with the heat equation on discrete space.
\subsection{Solving the Heat Equation}
In the case of the uniform Laplacian $R_{\mathrm{unif}} = \boldsymbol{1} \boldsymbol{1}^T/S - \mathrm{id}$, the heat equation has the closed form solution
\begin{equation}\label{equ:uniform_perturbation_estimation}
    q_t(y\vert x=a) = e^{-t}\delta(y, a) + \frac{1-e^{-t}}{S}\sum_{k=1}^S \delta(y, k)\,.
\end{equation}
Practically speaking, this perturbation remains in its state with probability $e^{-t}$ and samples uniformly from the state space otherwise. The case of the cyclical and ordinal structure is more delicate. We first note that the heat equation can be solved in terms of its eigenvalue expansion $\exp(Rt) = \mathbf{V} \exp(\Lambda t) \mathbf{V}^\ast$, where $\Lambda$ is the matrix with the eigenvalues $\lambda_p$ along its diagonal and $\mathbf{V}$ is a matrix of orthogonal eigenvectors with each column containing the corresponding eigenvector $\mathbf v_p$. The perturbation for $R^{\mathrm{cyc}}$ and $R^{\mathrm{ord}}$ can then be computed by means of a discrete Fourier transform:
\begin{restatable}[]{proposition}{restateproposition}\label{prop:SolutionHeatEquation}
    Assume the density $q_t(b \vert a) := q_t(y= b\vert x= a)$ is defined by the rate matrices $R^{\mathrm{cyc}}$ or $R^\mathrm{ord}$. The transition density for all $a, b \in \{1, 2, \dots, S\}$ is given by
\begin{align}
    q^\mathrm{cyc}_t(b\vert a) & = \frac{1}{S}\sum_{p=1}^S \exp(2\pi \mathrm i b\omega^\mathrm{cyc}_p)\, \exp\big((2\cos(2\pi \omega^\mathrm{cyc}_p) - 2) t\big)\,\exp(-2\pi \mathrm i a\omega^\mathrm{cyc}_p) \label{eq:cyclical_perturbation_estimation} \\
    q^\mathrm{ord}_t(b\vert a) & = \frac{2}{S} \sum_{p=1}^S \frac{1}{z_p}\, \cos((2b-1)\pi \omega^\mathrm{ord}_p) \,\exp\big((2\cos(2\pi \omega^\mathrm{ord}_p) - 2) t\big)\, \cos((2a-1)\pi\omega^\mathrm{ord}_p)\, \label{eq:ordinal_perturbation_estimation}
\end{align}
where $\omega_p^{\mathrm{cyc}} = (p-1)/S$ and $\omega_p^{\mathrm{ord}} = (p-1)/2S$, respectively, and $z_p = (2, 1, \dots, 1)$.
\end{restatable}

For the derivation, see \cref{der:eigenvalues}. Due to this result, the heat equation can be efficiently solved in parallel without requiring any sequential operations like multiple Euler steps. In addition, the transition matrices can be computed and saved in advance, thus reducing the computational complexity to the matrix multiplication with a batch of one-hot encoded data points.

\paragraph{Gaussian limit and choice of time parameter.} For tabular data sets the cardinality $S$ changes between different dimensions which raises the question how $t$ should be scaled with $S$. To answer this question we observe the following scaling limit of the perturbation:

\begin{restatable}[Scaling limit]{theorem}{restatetheoremtwo}\label{thm:ScalingLimit}
    Let $y_t \sim q_t(\cdot \vert x = \mu S)$ with $\mu \in \{1/S, 2/S, \dots, 1\}$ where $q_t$ is either the transition density of the cyclical or ordinal perturbation. Let $\varphi: \mathbb R\to (0, 1]$, where for all $n\in \mathbb Z$ and $x\in (0, 1]$ $\varphi^{\mathrm{cyc}}(n+x) = x$ and $\varphi^{\mathrm{ord}}(2n+x) = x$, $\varphi^{\mathrm{ord}}(2n+1+x) = -x$.  Then,
\begin{equation*}
    y_{S^2t}/S\xrightarrow{S\to\infty} \varphi(\xi_t) \quad \text{with}\quad \xi_t \sim \mathcal N(\mu, 2t)\,.
\end{equation*} 
\end{restatable}

Consequently, under the rescaling of time and space prescribed, the perturbation behaves independently of the state space size like a Gaussian with variance $2t$ that is reflected or periodically continued at the boundary states of $(0, 1]$. The phenomenon is visualised in \cref{fig:scaling_limit}. Based on this scaling limit we typically choose a quadratic rule $t = S^2t_{\mathrm{base}}$. Alternatively, we may choose a linear rule $t= S t_{\mathrm{base}}$ in which case the limit becomes a regular Gaussian on $\mathbb R_+$, thus recovering the Euclidean case from \citet{schroeder2023energy}. The theorem is proven in \cref{pf:UniversalScalingLimit}.

As a byproduct of this result we can also approximate the perturbation with discretised rescaled samples from a standard normal distribution and applying either periodic or reflecting mappings on perturbed states outside the domain. This may be computationally favourable for spaces of the form $\{1, \dots, S\}^d$ where the vocabulary size $S$ and dimension of the state space $d$ grow very large.

\paragraph{Localisation to random grid.}
For unstructured categorical variables the uniform perturbation may introduce too much noise to inform the EBM about the correlations in the data set. In these cases, it can be beneficial to sample a random dimension $k \in \{1, \dots, d\}$ and apply a larger perturbation in this dimension, only. This effectively means to replace the product of categorical distributions with a mixture perturbation
\begin{equation*}
    q_t(\y\vert \x) = \prod_{k = 1}^d  q_t(y_k \vert x_k) \rightarrow \frac{1}{d}\sum_{k = 1}^d q_t(y_k \vert x_k)
\end{equation*}
In our experiments we only consider the case of perturbing the randomly chosen dimension uniformly. We call this grid perturbation due to connections with concrete score matching \citep{Meng2022concreteSM}. The resulting loss can be understood as a variation of pseudo-likelihood estimation.
\paragraph{Special case of binary state space.}
In the special case of $\mathcal X = \{0, 1\}^d$, the structures of the cyclical, ordinal, and uniform graph coincide, and the perturbation $q_t(\mathbf y\vert\mathbf x)$ becomes the product of identical Bernoulli distributions with parameter $\varepsilon = 0.5(1 - \mathrm{e}^{-t})$. We also explore the grid perturbation which assumes that a dimension is selected at random and the entry is flipped deterministically from zero to one or one to zero. 
For details, see \cref{subsec:Binary}. 
\subsection{Estimation of the Contrastive Potential}
The final challenge in turning energy discrepancy into a practical loss function lies in the estimation of the contrastive potential $U_q$. We use the fact that for a symmetric rate matrix $R$, the induced perturbation is symmetric as well, i.e. $q_t(\mathbf y\vert \mathbf x) = q_t(\mathbf x \vert \mathbf y)$. Thus, we first write the contrastive potential as an expectation $U_q (\y) = -\log \sum_{\x\in\mathcal X}  \exp(-U(\x)) q(\y\vert\x) = -\log \mathbb{E}_{q(\mathbf x\vert \mathbf y)} \left[\exp(-U(\x))\right]$ and subsequently approximate the energy discrepancy loss as in \citet{schroeder2023energy} as
$\mathcal L_{q, M, w}(U) := \frac{1}{N} \sum_{i=1}^N \log\left(w+ \sum_{j= 1}^M \exp(U(\x^i) - U({\x}_-^{i,j})\right)$
with $\x^i \sim p_\data$, $\y^i \sim q_t(\cdot | \x^i)$, and ${\x}_-^{i,j} \sim q_t(\cdot\vert \mathbf \y^i)$.
The details of the training procedure are provided in \cref{sec:app-ebm-tabular}.

\section{Related Work}
\textbf{Contrastive loss functions.}
Our work is based on energy discrepancies first introduced in \citep{schroeder2023energy}. Energy discrepancies are equivalent to certain types of KL contraction divergences whose theory was studied in \citet{Lyu2011KLContractions}, however, without proposing a training algorithm for EBM's. On Euclidean data, ED is related to diffusion recovery-likelihood \citep{gao2020learning} which uses a CD-type training algorithm. For a masking perturbation, ED estimation can be understood as a Monte-Carlo approximation of pseudo-likelihood \citep{besagpseudolikelihood}. Furthermore, the structure of the stabilised energy discrepancy loss shares similarities with other contrastive losses such as \citet{ConditionalNoiseContrastiveEstimation,gutmann2010noise,Oord2018RepresentationLearning, foster2020unified} due to their close connection to the Kullback-Leibler divergence.

\textbf{Discrete diffusion models.}
We extend the continuous time Markov chain framework introduced and developed in \citet{campbell2022continuous, campbell2024generative, lou2023discrete} and provides a geometric interpretation thereof. Similar to us, \citet{kotelnikov2023tabddpm} defines a flow on mixed state spaces as the product of a Gaussian and a categorical flow, utilising multinomial flows \citep{hoogeboom2021argmax}. Our work has connections to concrete score matching \citep{Meng2022concreteSM} through the usage of neighbourhood structures to define a replacement of the continuous score function.

\textbf{Contrastive divergence and sampling.}
Contrastive divergence (CD) is commonly utilised for training energy-based models in continuous spaces with Langevin dynamics \citep{xie2016theory,xie2018cooperative,xie2022tale,du2020improved,xiao2020vaebm}. In discrete spaces, EBM training heavily relies on CD methods as well, which is a major driver for the development of discrete sampling strategies. The standard Gibbs method was improved by \citet{zanella2020informed} through locally informed proposals. This method was extended to include gradient information \citep{grathwohl2021oops} to drastically reduce the computational complexity of flipping bits in several places \citep{sun2022path, emami2023plug, sun2022optimalscaling}. Moreover, a discrete version of Langevin sampling was introduced based on this idea \citep{zhang2022langevin,rhodes2022enhanced,sun2023discrete}. Consequently, most current implementations of contrastive divergence use multiple steps of a gradient-based discrete sampler. Alternatively, EBMs can be trained using generative flow networks which learn a Markov chain that construct data optimising the energy as reward function \citep{zhang2022generative}.

\textbf{Other training methods and applications of EBMs for discrete and mixed data.}
A sampling-free approach for training binary discrete EBMs is ratio matching \citep{hyvarinen2007some,lyu2012interpretation, liu2023RMwGGIS}. \cite{dai2020learning} propose to apply variational approaches to train discrete EBMs instead of MCMC. \cite{eikemaapproximate} replace the widely-used Gibbs algorithms with quasi-rejection sampling to trade off the efficiency and accuracy of the sampling procedure. The perturb-and-map \citep{papandreou2011perturb} is also recently utilised to sample and learn in discrete EBMs \citep{lazaro2021perturb}. \citet{tran2011mixed} introduce mixed-variate restricted Boltzmann machines for energy-based modelling on mixed state spaces. Deep architectures, on the other hand, have been mostly limited to a single categorical target variable which is modelled via a classifier \citep{grathwohl2019your}.
Moreover, \cite{ou2022learning} apply discrete EBMs on set function learning, in which the discrete energy function is learned with approximate marginal inference \citep{domke2013learning}.
\section{Experiments}\label{sec:experiments}
To evaluate our proposed approach, we conduct experiments across diverse scenarios: i) estimating probability distributions on discrete data; ii) handling mixed-state features in tabular data; and iii) modelling binary images. We also explore Ising model training and graph generation in binary spaces, but leave the detailed evaluation of these in \cref{sec-appendix-experiments}.

\begin{figure}[!t]
    \centering
    \begin{minipage}[t]{0.05\linewidth}
        \centering
        \raisebox{-.5\height}{\includegraphics[width=.48in]{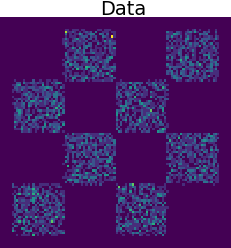}}
    \end{minipage}
    \begin{minipage}[t]{0.25\linewidth}
        \centering
        \includegraphics[width=.48in]{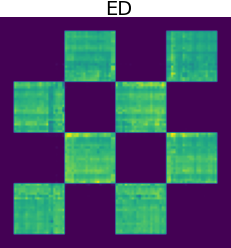}
        \includegraphics[width=.48in]{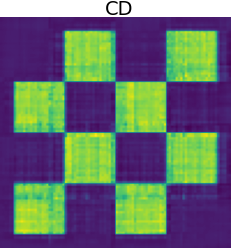}
        \includegraphics[width=.48in]{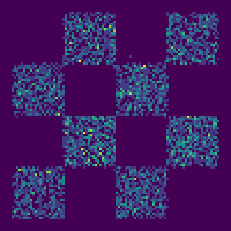}
        \includegraphics[width=.48in]{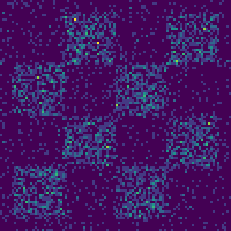}
    \end{minipage}
    \begin{minipage}[t]{0.05\linewidth}
        \centering
        \raisebox{-.5\height}{\includegraphics[width=.48in]{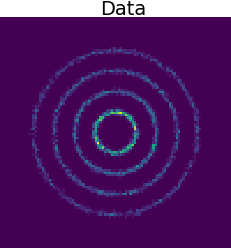}}
    \end{minipage}
    \begin{minipage}[t]{0.25\linewidth}
        \centering
        \includegraphics[width=.48in]{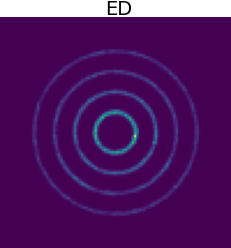}
        \includegraphics[width=.48in]{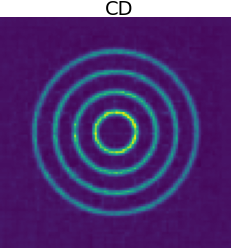}
        \includegraphics[width=.48in]{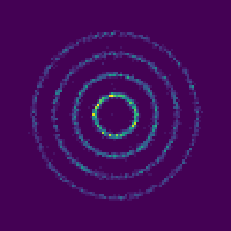}
        \includegraphics[width=.48in]{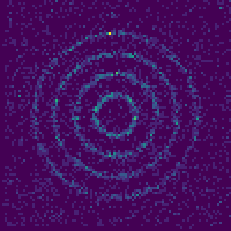}
    \end{minipage}
    \begin{minipage}[t]{0.05\linewidth}
        \centering
        \raisebox{-.5\height}{\includegraphics[width=.48in]{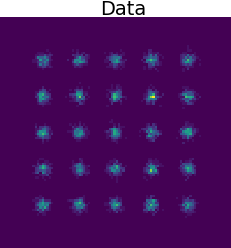}}
    \end{minipage}
    \begin{minipage}[t]{0.25\linewidth}
        \centering
        \includegraphics[width=.48in]{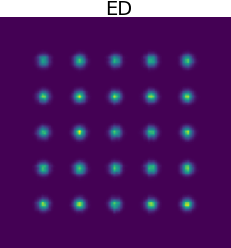}
        \includegraphics[width=.48in]{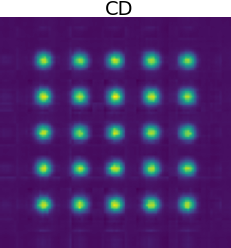}
        \includegraphics[width=.48in]{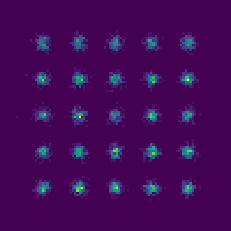}
        \includegraphics[width=.48in]{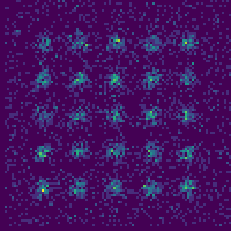}
    \end{minipage}
    \caption{Comparison of energy discrepancy and contrastive divergence on the dataset with $16$ dimensions and $5$ states. Rows $1$ and $2$ show the estimated density and synthesised samples, respectively.}
    \label{fig:exp-synthetic-5states}
\end{figure}

\subsection{Discrete Density Estimation}
We first demonstrate the effectiveness of energy discrepancy on density estimation using synthetic discrete data. Following \cite{dai2020learning}, we initially generate 2D floating-point data from several two-dimensional distributions. Each dimension of the data is then converted into a $16$-bit Gray code, resulting in a dataset with $32$ dimensions and $2$ states. To construct datasets beyond binary cases, we follow \cite{zhang2024formulating} and transform each dimension into $8$-bit $5$-base code and $6$-bit decimal code. This process creates two additional datasets: one with $16$ dimensions and $5$ states, and another with $12$ dimensions and $10$ states. The experimental details are given in \cref{appendix-sec-discrete-density-estimation}.

\cref{fig:exp-synthetic-5states} illustrates the estimated energies as well as samples that are synthesised with Gibbs sampling for energy discrepancy (ED) and contrastive divergence (CD) on the dataset with $16$ dimensions and $5$ states. It can be seen that ED excels at capturing the multi-modal nature of the distribution, consistently learning sharper energy landscape in the data support compared to CD. This coincides with the previous observations in continuous spaces \citep{schroeder2023energy}, suggesting ED's advantage in handling complex data structures.
For more results of additional datasets with $5$ and $10$ states, we deferred them to \cref{fig:appendix-exp-synthetic-5states,fig:appendix-exp-synthetic-10states}, respectively.

For binary cases with $2$ states, we compare our approaches to three baselines: PCD \citep{tieleman2008training}, ALOE+ \citep{dai2020learning}, and EB-GFN \citep{zhang2022generative}. 
In \cref{tab:synthetic_nll,tab:synthetic_mmd}, we quantitatively evaluate different methods by evaluating the negative log-likelihood (NLL) and the exponential Hamming MMD \citep{gretton2012kernel}, respectively. 
We observe that energy discrepancy outperforms the baseline methods in most settings, but without relying on MCMC simulations (as in PCD) or the training of additional variational networks (as in ALOE and EB-GFN). This performance gain is likely explained by the good theoretical guarantees of energy discrepancy for well-posed estimation tasks. In contrast, the baselines introduce biases due to their reliance on variational proposals and short-run MCMC sampling that may not have converged.

\subsection{Tabular Data Synthesising}
In this experiment, we assess our methods on synthesising tabular data, which presents a challenge due to its mix of numerical and categorical features, making it more difficult to model compared to conventional data formats. To demonstrate the efficacy of energy discrepancies, we first conduct experiments on synthetic examples before proceeding to real-world tabular data. 
Additional details regarding the experimental setup are deferred to \cref{appendix-sec-tabular}.

\begin{wrapfigure}{r}{0.4\linewidth}
\centering
\vspace{-6mm}
\begin{subfigure}{0.3\linewidth}
    \centering
    \includegraphics[width=1.\textwidth]{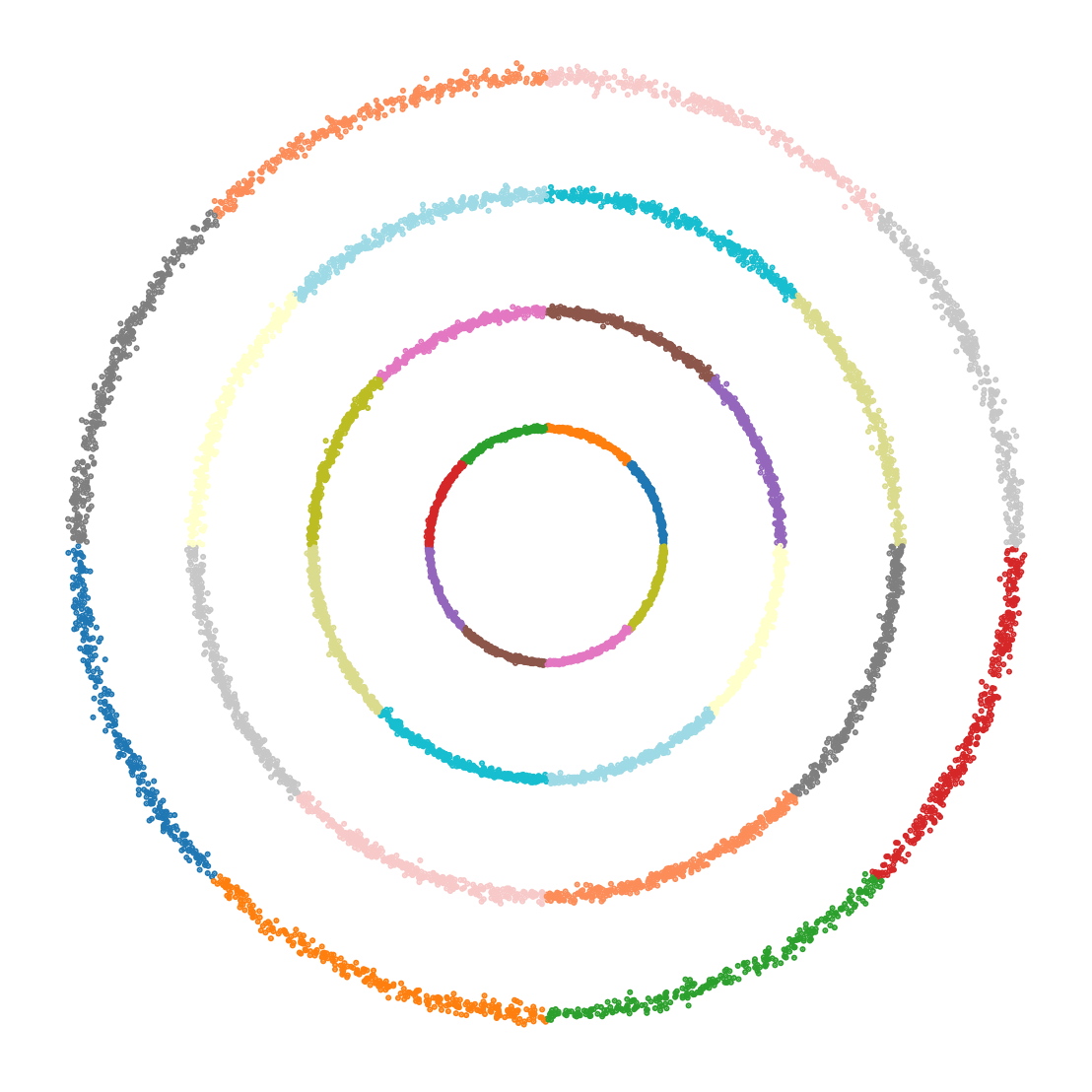}
    \vspace{-6mm}
    \caption*{{Data}}
\end{subfigure}
\begin{subfigure}{0.3\linewidth}
    \centering
    \includegraphics[width=1.\textwidth]{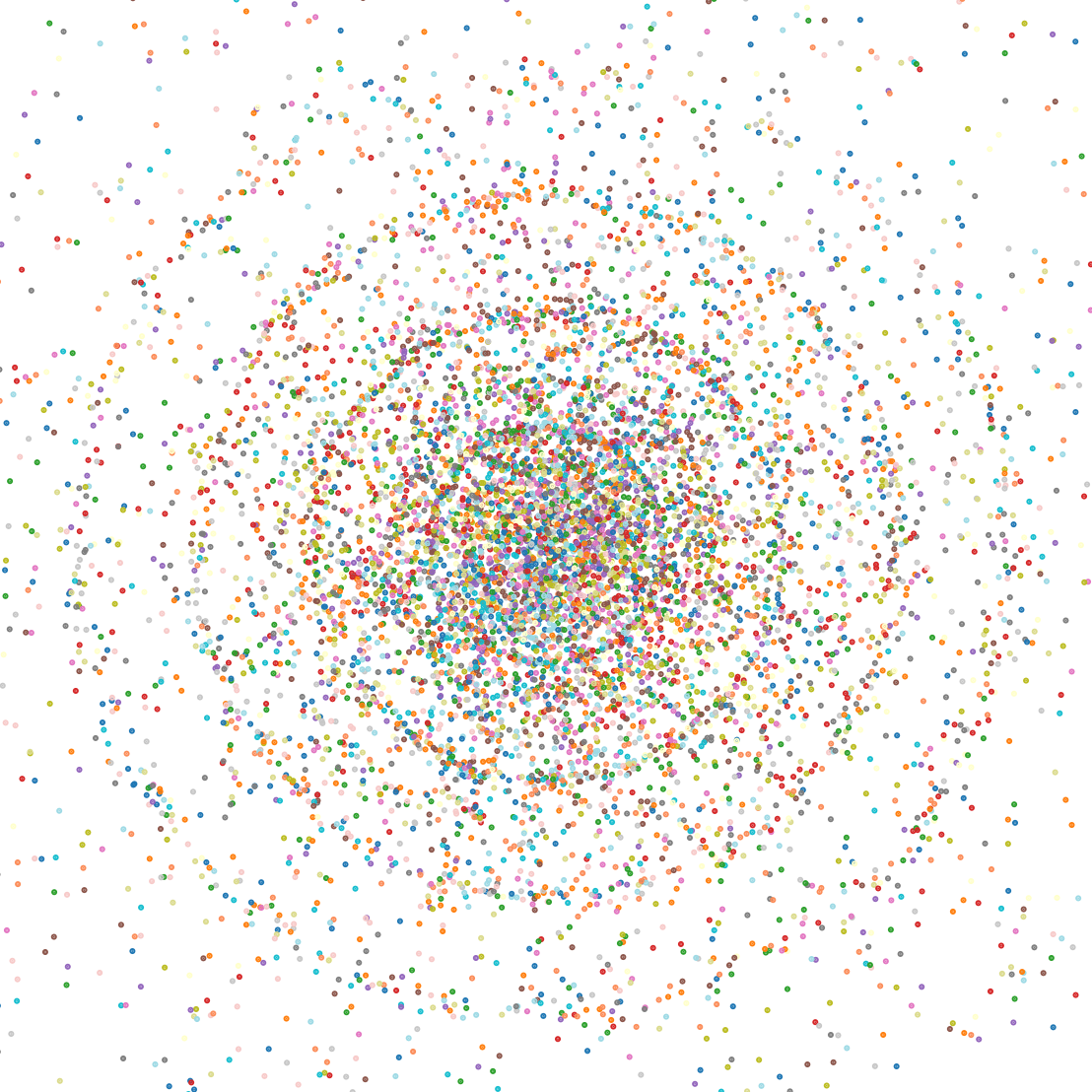}
    \vspace{-6mm}
    \caption*{{CD}}
\end{subfigure}
\begin{subfigure}{0.3\linewidth}
    \centering
    \includegraphics[width=1.\textwidth]{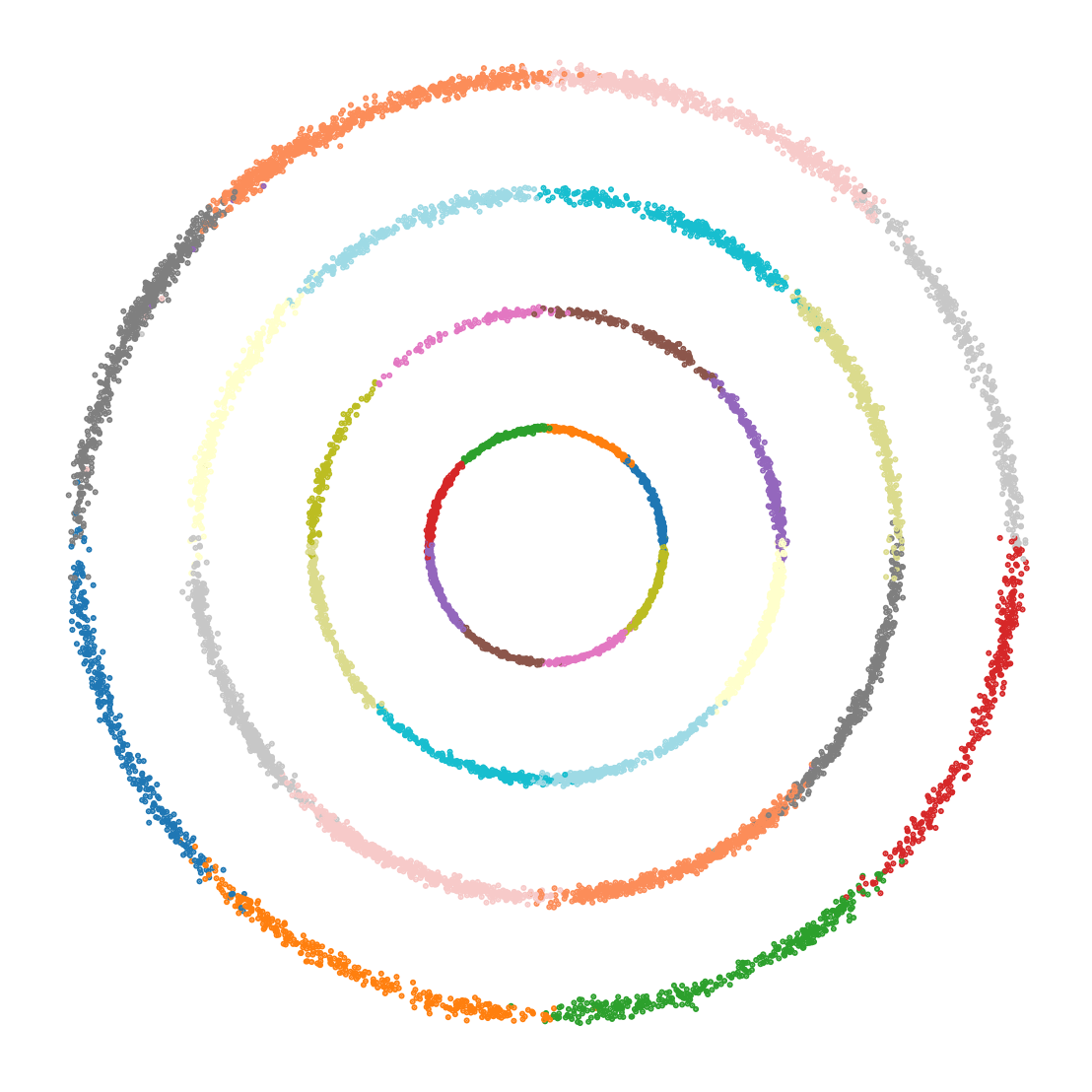}
    \vspace{-6mm}
    \caption*{{ED}}
\end{subfigure}
\caption{Comparison of the energy discrepancy and contrastive divergence on the synthetic tabular datasets.}
\label{fig:tabular_toy_compare}
\vspace{-4mm}
\end{wrapfigure}
\paragraph{Synthetic Dataset.} 
We first showcase the effectiveness of our methods on mixed data types by learning EBMs on a synthetic ring dataset. The dataset consists of four columns, with the first two columns indicating numerical coordinates of data points. The third column categorizes data points into four circles whereas the last column specifies the 16 colours each data point could be classified into. Therefore, each row in the tabular contains $2$ numerical features and $2$ categorical features.

To train an EBM on a dataset comprising mixed types of data, we employ either contrastive divergence or energy discrepancy. For CD, we adopt a strategy involving a replay buffer in conjunction with a short-run MCMC using $20$ steps. Specifically, we utilise Langevin dynamics and Gibbs sampling for numerical and categorical features, respectively. In the case of ED, we perturb the numerical features with a Gaussian perturbation and the categorical features with grid perturbation. 
\cref{fig:tabular_toy_compare} illustrates the results of synthesised samples generated from the learned energy using Gibbs sampling.
These findings align with those depicted in \cref{fig:exp-synthetic-5states}, where CD struggles to capture a faithful energy landscape, leading to synthesized samples potentially lying outside the data distribution support.
Instead, by leveraging a combination of perturbation techniques tailored to the data types present, ED offers a more robust and reliable framework for training EBMs in mixed state spaces.

\begin{table}[!t] 
    \centering
    \caption{Results on real-world datasets.}
    \label{tbl:exp-mle}
    \small
    \setlength{\tabcolsep}{1.6mm}{
    \begin{threeparttable}
    {
        \begin{tabular}{lcccccc|cc}
            \toprule[0.8pt]
            \multirow{2}{*}{Methods} & {\textbf{Adult}} &{\textbf{Bank}} & \textbf{Cardio} & {\textbf{Churn}} &   {\textbf{Mushroom}} & {\textbf{Beijing}} & {\textbf{Avg. Rank}} \\
            \cmidrule{2-8} 
            & AUC $\uparrow$ & AUC $\uparrow$ &  AUC $\uparrow$ &  AUC $\uparrow$  & AUC $\uparrow$ &  RMSE $\downarrow$ & $-$ \\
            \midrule 
            Real & $.927${\tiny$\pm.000$} & $.935${\tiny$\pm.002$} & $.834${\tiny$\pm.001$}  & $.819${\tiny$\pm.001$}  & $1.00${\tiny$\pm.000$} & $.423${\tiny$\pm.003$} & $-$  \\
            \midrule
            CTGAN & $.861${\tiny$\pm.005$} & $.774${\tiny$\pm.006$} & $.787${\tiny$\pm.002$} & $.792${\tiny$\pm.003$}    & $.781${\tiny$\pm.007$} & $1.01${\tiny$\pm.038$} & $6.33$ \\
            TVAE  & $.873${\tiny$\pm.001$} & $.868${\tiny$\pm.002$} & $.676${\tiny$\pm.009$} & $.793${\tiny$\pm.006$}    & $.999${\tiny$\pm.000$} & $1.05${\tiny$\pm.012$} & $5.17$ \\
            TabCD & $.619${\tiny$\pm.026$} & $.604${\tiny$\pm.021$} & $.765${\tiny$\pm.008$} & $.584${\tiny$\pm.021$}    & $.561${\tiny$\pm.048$} & $1.06${\tiny$\pm.037$} & $8.83$ \\
            TabDDPM & $.910${\tiny$\pm.001$} & $.922${\tiny$\pm.001$} & $.801${\tiny$\pm.001$} & $.806${\tiny$\pm.007$}    & $.999${\tiny$\pm.000$} & $.556${\tiny$\pm.005$} & $ 1.5$ \\
            \midrule
            TabED-Uni & $.884${\tiny$\pm.003$} & $.842${\tiny$\pm.013$} & $.786${\tiny$\pm.002$} & $.810${\tiny$\pm.008$}    & $.998${\tiny$\pm.001$} & $1.04${\tiny$\pm.013$} & $3.83$ \\
            TabED-Grid & $.833${\tiny$\pm.003$} & $.831${\tiny$\pm.004$} & $.791${\tiny$\pm.001$} & $.803${\tiny$\pm.007$}    & $.985${\tiny$\pm.005$} & $.978${\tiny$\pm.015$} & $4.83$ \\
            TabED-Cyc & $.831${\tiny$\pm.005$} & $.823${\tiny$\pm.007$} & $.796${\tiny$\pm.001$} & $.807${\tiny$\pm.007$}    & $.971${\tiny$\pm.004$} & $1.01${\tiny$\pm.024$} & $4.83$ \\
            TabED-Ord & $.853${\tiny$\pm.005$} & $.845${\tiny$\pm.004$} & $.792${\tiny$\pm.002$} & $.806${\tiny$\pm.004$}    & $.926${\tiny$\pm.010$} & $1.02${\tiny$\pm.017$} & $4.83$ \\
            TabED-Str & $.879${\tiny$\pm.004$} & $.819${\tiny$\pm.001$} & - & - & - & $.978${\tiny$\pm.012$} & $3.67$ \\
		\bottomrule[1.0pt] 
		\end{tabular}
    }
    \end{threeparttable}
    }
\end{table}

\paragraph{Real-world Dataset.}
We then evaluate our methods by benchmarking them against various baselines across $6$ real-world datasets. Following \cite{xu2019modeling}, we first split the real datasets into training and testing sets. The generative models are then learned on the real training set, from which synthetic samples of equal size are generated. This synthetic dataset is subsequently used to train a classification/regression XGBoost model, which is evaluated using the real test set.

We compare the performance, as measured by the AUC score for classification tasks and RMSE for regression tasks, against CTGAN, TVAE, \citep{xu2019modeling} and TabDDPM \citep{kotelnikov2023tabddpm} baselines which utilise generative adversarial networks, variational autoencoders, and denoising diffusion probabilistic models, respectively. The results are reported in \cref{tbl:exp-mle}. Here, TabED-Str refers to an ED loss for which the perturbation was chosen with prior knowledge about the structure of the modelled feature, i.e. ordinal and cyclical features were hand-picked. We do not report results for TabED-Str on the Cardio, Churn, and Mushroom datasets, since the state spaces only consist of unstructured features. To compute the average ranking we use the rank of TabED-Uni on these datasets since on unstructured features TabED-Uni and TabED-Str coincide.

The variants of ED show promising results on diverse datasets, thus demonstrating the suitability of ED for EBM training on mixed-state spaces. While TabDDPM outperforms the other approaches, TabED shows comparable performance to the CTGAN and TVAE baselines and outperforms both in average ranking. Furthermore, the contrastive divergence approach performs poorly which highlights its limitations in accurately modelling distributions on mixed state spaces. Surprisingly, the unstructured perturbation TabED-Uni performs slightly better than the structured approaches. This may partially be attributed to the fact that the state spaces of the discrete features are relatively small. Consequently, the uniform perturbation might be a good approximation of maximum likelihood estimation in agreement with \cref{TheoremMLElimit}, while not producing high-variance gradients on these specific datasets.

\begin{wrapfigure}{r}{0.6\linewidth}
\centering
\vspace{-2mm}
\begin{subfigure}{0.48\linewidth}
    \centering
    \includegraphics[width=1.\textwidth]{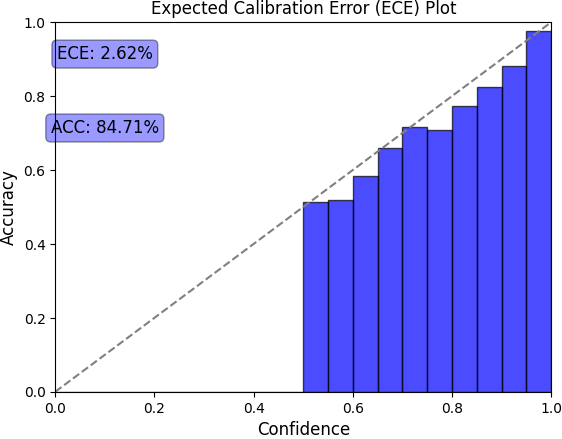}
    \vspace{-3mm}
\end{subfigure}
\begin{subfigure}{0.48\linewidth}
    \centering
    \includegraphics[width=1.\textwidth]{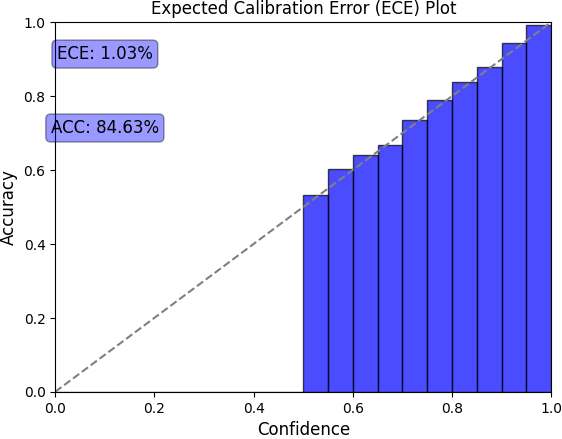}
    \vspace{-3mm}
\end{subfigure}
\caption{Calibration results comparison between the baseline (left) and energy discrepancy (right) on the adult dataset.}
\label{fig:tabular_adult_ece}
\vspace{-4mm}
\end{wrapfigure}
\paragraph{Improving Calibration.}
Despite the improving accuracy of neural-network-based classifiers in recent years, they are also becoming increasingly recognised for their tendency to exhibit poor calibration due to overconfident outputs \citep{guo2017calibration,mukhoti2020calibrating}. Since energy-based model on mixed state spaces can capture the likelihood of tuples of features and target labels, they implicitely quantify the confidence in a prediction and
can be adapted into classifiers with better calibration than deterministic methods. This opens up a new avenue for applying EBMs in deterministic tabular data modelling methods.

Let $y$ and $\x$ be the target label and the rest features in the tabular data, an EBM $U_\theta (\x, y)$ learned on the joint probability $p_\data (\x,y)$ can be transformed into a deterministic classifier: $p_{\text{EBM}}(y|\x) \propto \exp(-U_\theta (\x, y))$. As a baseline for comparison, we additionally train a classifier $p_{\text{CLF}}(y|\x)$ with the same architecture by maximising the conditional likelihood: $\mathbb{E}_{p_\data} [\log p_{\text{CLF}}(y|\x)]$.
Results on the adult dataset can be seen in \cref{fig:tabular_adult_ece}.
We find that the EBM and the baseline exhibit comparable accuracy. However, the baseline model is less calibrated, generating over-confident predictions. In contrast, the EBM learned through ED achieves better calibration, as evidenced by lower expected calibration error \citep{guo2017calibration}. Further details and results are provided in \cref{appendix-sec-tabular}.
\begin{table}[t]
\small
\centering
\caption{Experimental results for discrete image modelling. We report the negative log-likelihood (NLL) on the test set for different models. The results of Gibbs, GWG, and DULA are taken from \cite{zhang2022langevin}, and the result of EB-GFN is from \cite{zhang2022generative}.}
\label{tab:image_logll}
\begin{tabular}{l|cccccc}
\toprule 
Dataset $\backslash$ Method & Gibbs & GWG & EB-GFN & DULA & \EDB & \EDG \\
\midrule 
Static MNIST & $117.17$ & $\textbf{80.01}$ & $102.43$ & $80.71$ & $96.11$ & $90.61$  \\
Dynamic MNIST & $121.19$ & $\textbf{80.51}$ & ${105.75}$ & $81.29$ & $97.12$ & $90.19$ \\
Omniglot & $142.06$ & ${94.72}$ & ${112.59}$ & $145.68$ & $97.57$ & $\bf 93.94$ \\
\bottomrule
\end{tabular}
\vspace{-2mm}
\end{table}
\subsection{Discrete Image Modelling}
In this experiment, we evaluate our methods on high-dimensional binary spaces. Following the settings in \cite{grathwohl2021oops}, we conduct experiments on various image datasets and compare against contrastive divergence using various sampling methods, namely vanilla Gibbs sampling, Gibbs-With-Gradient \cite[GWG]{grathwohl2021oops}, Generative-Flow-Network \cite[GFN]{zhang2022generative}, and Discrete-Unadjusted-Langevin-Algorithm \cite[DULA]{zhang2022langevin}. The training details are provided in \cref{appendix-sec-discrete-image-modelling}. After training, annealed importance sampling \citep{neal2001annealed} is employed to estimate the negative log-likelihood (NLL).

\cref{tab:image_logll} displays the NLLs on the test dataset.  It is evident that energy discrepancy achieves comparable performance to the baseline methods on the Omniglot dataset. Despite the performance gap compared to the contrastive divergence methods on the MNIST dataset, energy discrepancy stands out for its efficiency, requiring only $M$ evaluations of the energy function in parallel (see \cref{tab:image_time_complexity} for the comparison of running time complexity). This represents a significant computational reduction compared to contrastive divergence, which lacks the advantage of parallelisation and involves simulating multiple MCMC steps.
Additionally, our methods show superiority over CD-1 by a substantial margin, as demonstrated in \cref{tab:image_gen_cd-n}, affirming the effectiveness of our approach.
For further insights, we provide visualisations of the generated samples in \cref{fig:sample-ebm-appendix}.

\section{Conclusion and Limitations}
In this paper we extend the training of energy-based models with energy discrepancy to discrete and mixed state spaces in a systematic way. We show that the energy-based model can be learned jointly on continuous and discrete variables and how prior assumptions about the geometry of the underlying discrete space can be utilised in the construction of the loss. Our method achieves promising results on a wide range of discrete modelling applications at a significantly lower computational cost than MCMC-based approaches. To the best of our knowledge, our approach is also the first working training method for energy-based models on tabular data sets, unlocking a wide range of inference applications for tabular data sets beyond the scope of classical joint energy-based models.

\textbf{Limitations:} Similar to prior work on energy discrepancy in continuous spaces \citep{schroeder2023energy}, our training method is sensitive to the assumption that the data distribution is positive on the whole state space. While our method scales to high-dimensional datasets like binary image data, where the positiveness of the data distribution is assumed to be violated due to the manifold hypothesis, the large difference between intrinsic and ambient dimensionality poses challenges to our approach and may explain why energy discrepancy cannot match the performance of contrastive divergence with a large number of MCMC steps on binary image data.

\textbf{Broader Impact:} In principle, our method can be used for imputation and prediction in tabular data sets and can thus have discriminating or excluding effects if used irresponsibly. 

\textbf{Outlook:} For future work, we are interested in extensions to highly structured types of data such as molecules, text, or data arising from networks. So far, our work only considers cyclical and ordinal structures on the discrete space, while incorporating more complex structures as prior information into the rate function may be beneficial. Furthermore, interesting downstream applications ranging from table imputation with confidence bounds, simulation-based inference involving discrete variables, or reweighting of language models with residual EBMs have been left unexplored in this work.

\subsection*{Author Contributions}
TS and ZO conceived the project idea to use an ED loss for the training of EBMs on discrete and mixed data. TS devised the main conceptual ideas, developed the theory, conducted the proofs and implemented the ED loss. ZO contributed to the conceptual ideas and designed and carried out the experiments. ABD supervised the conceptualisation and execution of the research project and contributed proof ideas. ZO, YL, and ABD checked derivations and proofs. 
TS and ZO equally contributed to the writing under the supervision of YL and ABD.

\subsection*{Acknowledgements}
TS was supported by the EPSRC-DTP scholarship partially funded by the Department of Mathematics, Imperial College London. ZO was supported by the Lee Family Scholarship. 
We thank the anonymous reviewer for their comments.


\newpage
\bibliography{main}
\bibliographystyle{icml2022}

\newpage 
\appendix

\begin{center}
\LARGE
\textbf{Appendix for ``Energy-Based Modelling for Discrete and Mixed Data via Heat Equations \\on Structured Spaces''}
\end{center}

\etocdepthtag.toc{mtappendix}
\etocsettagdepth{mtchapter}{none}
\etocsettagdepth{mtappendix}{subsection}
{\small \tableofcontents}
\section{Proofs of the Main Results}\label{sec:proofs}
\subsection{Proof of \cref{TheoremMLElimit}}\label{pf:TheoremMLELimit}
\restatheoremone*
\begin{proof}
For $\mathcal X = \{1, 2, \dots, S\}$ and two probability distributions $p_1$ and $p_2$ on $\mathcal X$ define the marginal distributions
\begin{equation}
    Q_tp_1(y) = \sum_{x\in \mathcal X} q_t(y\vert x) p_1(x) \,, \quad Q_tp_2(y) = \sum_{x\in \mathcal X} q_t(y\vert x) p_2(x)\,.
\end{equation}
By the data-processing inequality it holds for all $p_1, p_2$ with $\infdiv{p_1}{p_2}<\infty$ that
\begin{equation}
    \infdiv{Q_tp_1}{Q_tp_2} \leq c(t) \infdiv{p_1}{p_2} \quad \text{with} \quad c(t)\leq 1\,.
\end{equation}
We are going to bound the contraction rate $c(t)$ by firstly making use strong data processing inequality \citet{raginsky2016strong} which states that the contraction rate can be bounded by the Dobrushin contraction coefficient
\begin{equation}
    c(t) \leq \theta(Q_t) = \sup_{x, x'} \lVert Q_t \delta_x - Q_t \delta_x' \rVert_{\mathrm{TV}}\,.
\end{equation}
Furthermore, since total variation is a metric, the contraction between two points $x, x'$ is upper bounded by the contraction towards the stationary distribution of the CTMC $\pi$:
\begin{equation}
    \sup_{x, x'}\lVert Q_t \delta_x - Q_t \delta_x' \rVert_{\mathrm{TV}} \leq 2\sup_{x} \lVert Q_t \delta_x -  \pi \rVert_{\mathrm{TV}}
\end{equation}
Next, we use \citet[Proposition 3]{diaconis1991geometric} which states that for any $x\in \mathcal X$ and for eigenvalues $0 = \lambda_1(R) \geq \lambda_2(R) \geq \dots \geq \lambda_S(R)$
\begin{equation}
    \lVert Q_t\delta_\x - \pi\rVert_{\mathrm{TV}}^2 \leq \frac{1}{4} \frac{1-\pi(\x)}{\pi(x)}\exp(-2\vert\lambda_2(R)\vert t)\,.
\end{equation}
The stationary distribution for $R^{\mathrm{ord}}, R^{\mathrm{cyc}}, R^{\mathrm{unif}}$ is given by the stationary distribution and hence $(1-\pi(\x))/\pi(x) \leq S$. Taking roots now yields
\begin{equation}
    c(t) \leq \sup_{x, x'}\lVert Q_t \delta_x - Q_t \delta_x' \rVert_{\mathrm{TV}} \leq \sqrt{S} \exp(-\vert\lambda_2(R)\vert t)\,.
\end{equation}
Finally, we conclude as in \citet{schroeder2023energy}:
\begin{align*}
    \infdiv{Q_t p_\data}{Q_t p_\theta} &= \sum_{y\in \mathcal X} \left(\log Q_tp_\data(y) - \log \frac{Q_t p_\theta(y)}{p_\theta(x)} - \log p_\theta(x)\right) Q_tp_\data(y) \\\notag
    &= z_t + \sum_{y\in \mathcal X} U_{q_t, \theta}(y) Q_tp_\data(y) - U_\theta(x) -\log p_{\theta}(x)
\end{align*}
with $U$ independent entropy term $z_t:= \sum_{y\in \mathcal X}Q_tp_\data(y)\log Q_tp_\data(y)$. After taking expectations with respect to $p_\data(x)$ we find
\begin{align*}
    0 &\leq \infdiv{Q_t p_\data}{Q_t p_\theta} \\\notag
    &= z_t + \sum_{y\in \mathcal X} U_{q_t, \theta}(y) Q_tp_\data(y) - \sum_{x\in \mathcal X} U_\theta(x) p_\data(x) - \sum_{x\in \mathcal X}  \log p_{\theta}(x) p_\data(x)\\\notag
    &= z_t - \ED_{q_t}(p_\data, U_\theta) - \mathbb E_{p_\data(x)}[\log p_\theta(x)] \leq c(t)\infdiv{p_\data}{p_\theta}
\end{align*}
\end{proof}

\subsection{Eigenvalue Decomposition of Rate Matrices for \cref{prop:SolutionHeatEquation}}\label{der:eigenvalues}
The rate matrices for the cyclical and for the ordinal graph structures have a similar structure and are referred to as circulant and tridiagonal matrices. The easiest method for deriving the eigenvalue decompositions consists in deriving recurrence relations for the characteristic polynomial. A systematic study of block circulant matrices can be found in \cite{tee2007eigenvectors} and a study of tridiagonal matrices was given in \cite{losonczi1992eigenvalues, yueh2005eigenvalues}. These more general results may be helpful when constructing perturbations for spaces with a more complex structure than the ones introduced in this work. We take already existing results and check that the desired results hold.
\restateproposition*
\begin{proof}
    In the circular case, the identity $R\mathbf v_p = \lambda_p \mathbf v_p$ reduces for a single row $a$ to
    \begin{equation}
        v_{p, a-1} + v_{p, a+1} - 2v_{p, a} = \lambda_p v_{p, a}
    \end{equation}
    For $ v_{p, a}= \exp(-2\pi \iu (p-1)a/S)/\sqrt{S}$ we find after dividing both sides by $v_{p, a}$:
    \begin{align}
        \lambda_p &= \exp\left( -2\pi \iu \frac{(p-1)(a-1) -(p-1)a}{S}\right) + \exp\left( -2\pi \iu \frac{(p-1)(a+1) -(p-1)a}{S}\right) - 2 \\\notag
        & = 2\cos\left(2\pi \frac{p-1}{S}\right) - 2 \,.
    \end{align}
    Furthermore, it is known that the Fourier basis is unitary, i.e. we have for $a, b \in \{1, 2, \dots, S\}$
    \begin{align}
        \sum_{p = 1}^S \bar v_{a, p} v_{b, p} &= \frac{1}{S} \sum_{p = 1}^S \exp\left( -2\pi \iu \frac{(a-b)(p-1)}{S}\right) = \delta(a, b)\,.
    \end{align}
    In the tridiagonal case, we have to study two cases. Firstly, we have to check in row $1$ using $2\cos(x)\cos(y) = \cos(x+y) + \cos(x-y)$:
    \begin{align}
        &v_{p, 2} - v_{p, 1} = \lambda_p v_{p, 1} = 2\cos\left(\frac{2\pi (p-1)}{2S}\right)\cos\left(\frac{\pi (p-1)}{2S}\right) - 2\cos\left(\frac{\pi (p-1)}{2S}\right) \\\notag
        & = \cos\left(\frac{2\pi (p-1)}{2S} + \frac{\pi (p-1)}{2S}\right) + \cos\left(\frac{2\pi (p-1)}{2S} - \frac{\pi (p-1)}{2S}\right) - 2\cos\left(\frac{\pi (p-1)}{2S}\right) \\\notag
        & = \cos\left(\frac{3\pi (p-1)}{2S}\right) - \cos\left(\frac{\pi (p-1)}{2S}\right)\,.
    \end{align}
    and equivalently for row $S$. In all other rows we have $v_{p, a+1} + v_{p, a-1} - 2v_{p, a} = \lambda_p v_{p, a}$ from
    \begin{equation}
      \frac{(2a-1)(p-1)}{2S} \pm \frac{2(p-1)}{2S}  = \frac{(2(a\pm1) - 1)(p-1)}{2S}\,.
    \end{equation}
\end{proof}

\subsection{Proof of Scaling limit in \cref{thm:ScalingLimit}}\label{pf:UniversalScalingLimit}
It is a typical generalisation of the central limit theorem that random walks attain Brownian motion as a universal scaling limit. We reproduce similar arguments for the law of the continuous time Markov chain. Without loss of generality we shift the state space by one and consider the state space $\{0, 1, \dots, S-1\}$ with cyclical and ordinal structure and let $y_t\sim q_t(\cdot \vert x= \mu S)$, where we always assume that the process is initialised at state $\mu S$. Furthermore, we introduce the process $z_t$ which is an unconstrained continuous time Markov chain on $\mathbb Z$ with rate matrix $R_{aa} = -2, R_{a, a+1}=1, R_{a, a-1} = 1$. The constrained process can then be described in terms of the unconstrained one: Let $\varphi_S: \mathbb Z \to \{0, 1, \dots, S-1\}$ with $\varphi_S(2nS + p) = p$ and $\varphi_S((2n + 1)S + p) = -p$ for for $p\in \{0, 1, \dots, S-1\}$ and $n\in \mathbb Z$. Then, $\varphi_S$ reflects the unconstrained process $z_t$ at the boundaries $0$ and $S-1$, i.e. $y_t^{\mathrm{ord}} = \varphi_S(z_t)$ and $y_t^{\mathrm{cyc}} = z_t\,\mathrm{mod}\, S $. For the unconstrained process we define the holding times, i.e. the random time intervals in which the process does not change state
\begin{equation}
    h_a = \inf_{t\geq 0}\{t : z_t\neq x = a\}
\end{equation}
and the jump process $N_t := \sum_{h\leq t} \delta(z_h \neq z_{h-})$, where $z_{h-} = \lim_{s\uparrow h} z_s$ which counts the number of state transitions up to time $t$. It is a standard result that the holding times are exponentially distributed \citep[Proposition 2.8]{anderson2012continuous} $h_a \sim \mathrm{Exp}(-R_{aa})$. Furthermore, since all holding times are identically distributed and $R_{aa} = -2$, the resulting jump process has Poisson distribution
\begin{equation}
    N_t\sim \mathrm{Poisson}(2t)\,.
\end{equation}
With these definitions, we can now first proof the Gaussian limit of $z_t$ and then derive the limit of $y_t$.
\restatetheoremtwo*
\begin{proof}
First, we write $z_t$ as the sum of independent increments starting at $x = \mu S$
\begin{equation}
    z_t = \mu S + \sum_{j=1}^{N_t} J_j\, , \quad p(J_j = 1) = p(J_j = -1) = 1/2\,.
\end{equation}
We can now compute the characteristic function $\chi_S(s) = \mathbb E[\exp(isz_t)]$ of $z_t$ rescaled in space and time. We have the following:
\begin{align}
    \chi_S(s) &= \mathbb E\left[\exp\left( \iu s \frac{y_{S^2t}}{S}\right)\right]\\\notag
    &= \exp(\iu s \mu)\mathbb E\left[\mathbb E\left[\exp\left( \iu s \sum_{j=1}^{N_{S^2t}} J_j /S\right) \right]\Bigg\vert N_{S^2t}\right]\\\notag
    &= \exp(\iu s \mu)\mathbb E\left[ \prod_{j=1}^{N_{S^2t}} \cos\left(\frac{s}{S}\right)\Bigg\vert N_{S^2t}\right] \\\notag
    &= \exp(\iu s \mu)\sum_{K = 0}^\infty \cos\left(\frac{s}{S}\right)^K \frac{(2S^2t)^K \exp(-2S^2t)}{K!} = \exp(\iu s \mu)\exp\left( 2S^2t\cos\left(\frac{s}{S}\right) - 2S^2t \right)
\end{align}
where we used that $\cos(x) = 1/2(\exp(\iu x) + \exp(-\iu x))$ in the third step, the Poisson distribution of $N_t$ in the fourth step, and the series expansion of the exponential function in the final step. Since $\cos(x) \approx 1 - 1/2 x^2$, we now have point-wise for any $s$ and due to the fact that $s/S \to 0$
\begin{equation}
    \chi_S(s) \xrightarrow{S\to\infty} \exp(\iu s\mu -ts^2)
\end{equation}
which is the characteristic function of a Gaussian with variance $2t$ and mean $\mu$. This proves the convergence in distribution of the rescaled unconstrained process $z_t$. Furthermore, for $\varphi^{\mathrm{ord}}$ and $\varphi^{\mathrm{cyc}}$ it holds $\varphi_S(z_t)/S = \varphi(z_t/S)$ for all $S\in \mathbb N$. Furthermore, $\varphi^{\mathrm{ord}}, \varphi^{\mathrm{cyc}}$ are continuous maps from $\mathbb R$ to $[0, 1)$ with reflecting or periodic boundary conditions. We thus have by the continuous mapping theorem
\begin{equation}
    \frac{y_{S^2t}}{S} = \frac{\varphi_S(z_{S^2t})}{S} = \varphi\left(\frac{z_t}{S}\right) \xrightarrow{S\to\infty} \varphi(\xi_t)
\end{equation}
with $\xi_t\sim \mathcal N(\mu, 2t)$.
\end{proof}
\begin{figure}
    \centering
    \includegraphics[width = \textwidth]{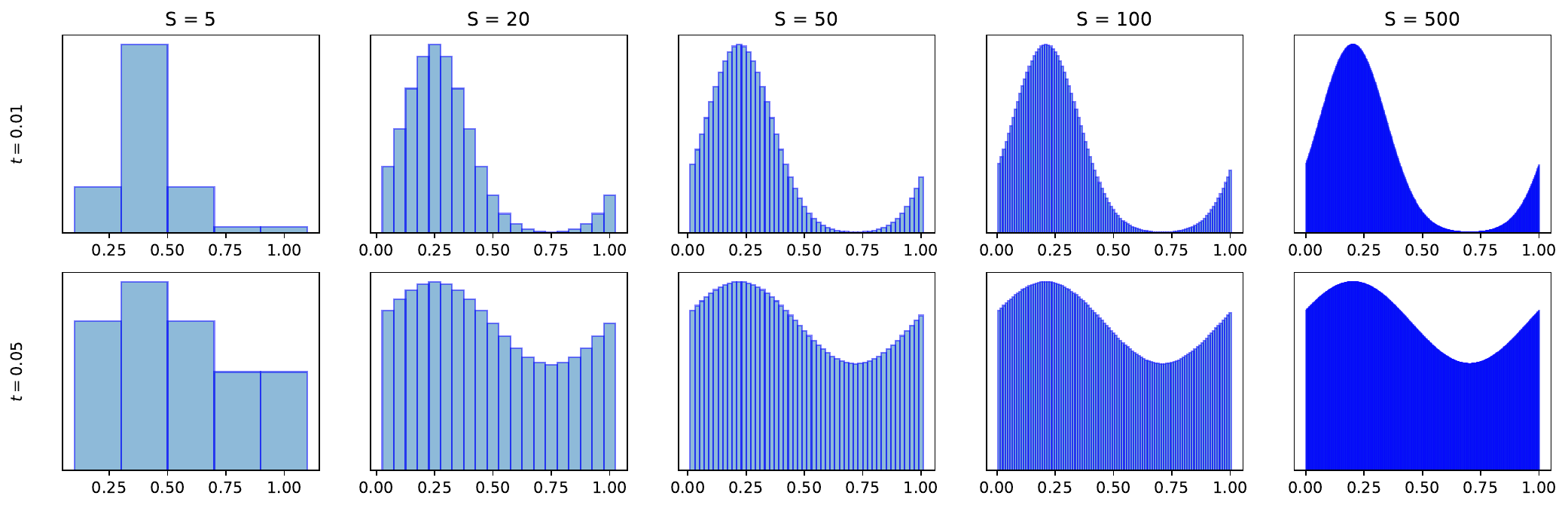}
    \includegraphics[width = \textwidth]{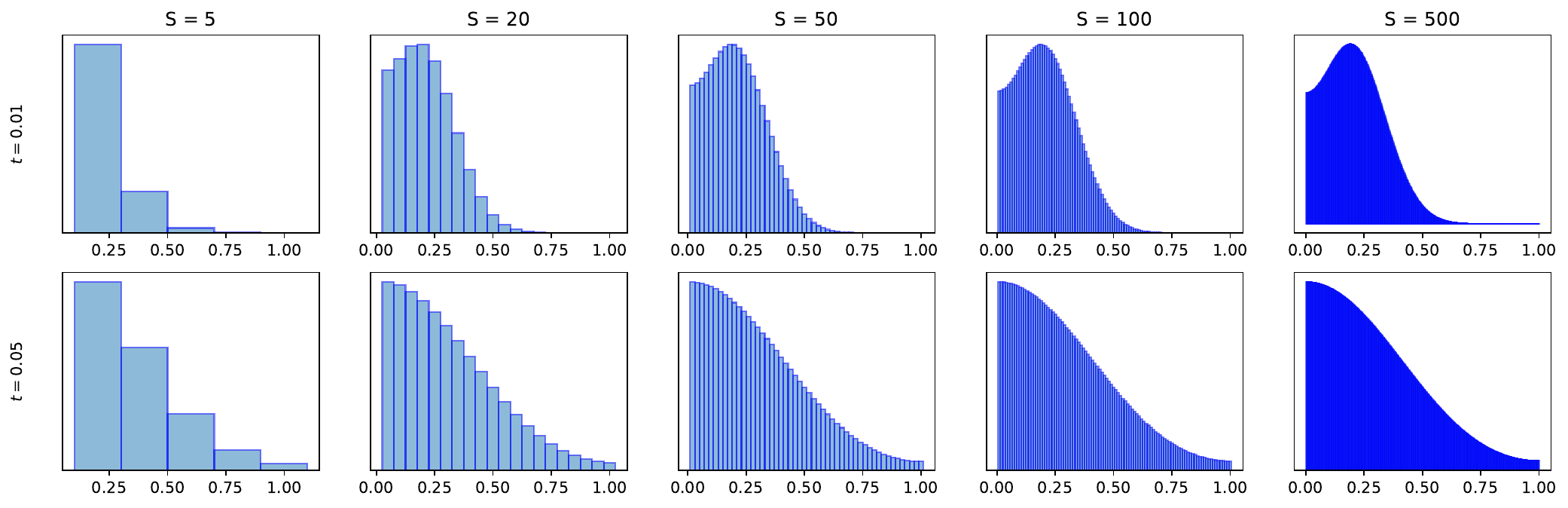}
    \includegraphics[width = \textwidth]{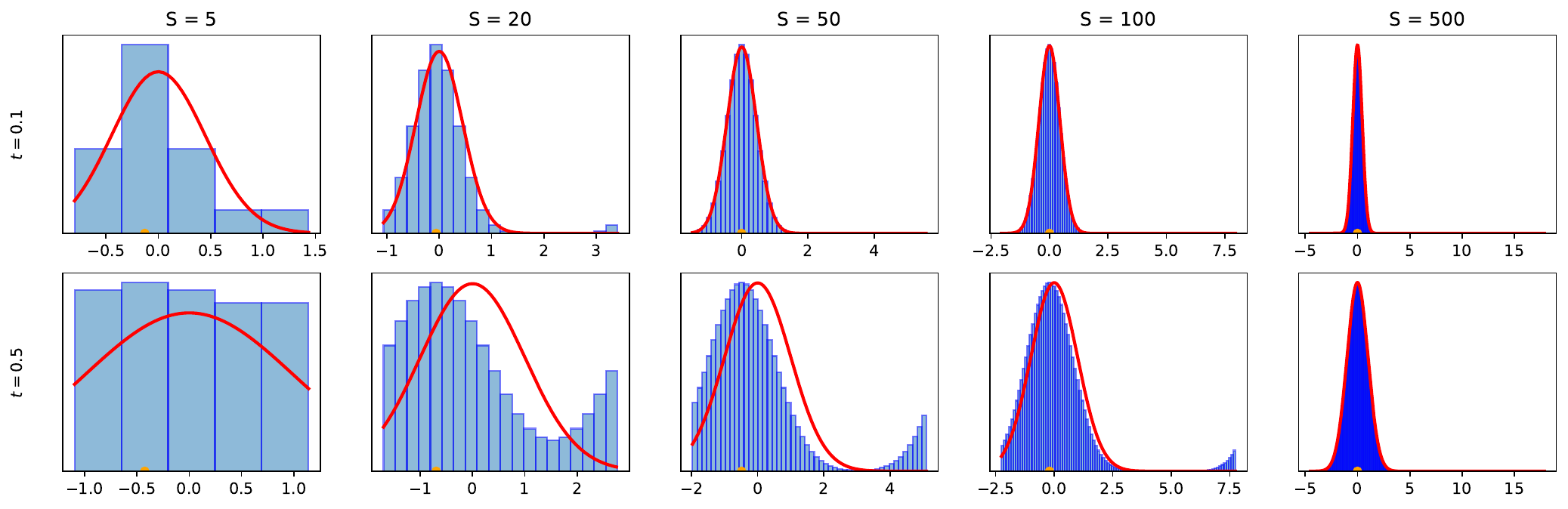}
    \includegraphics[width = \textwidth]{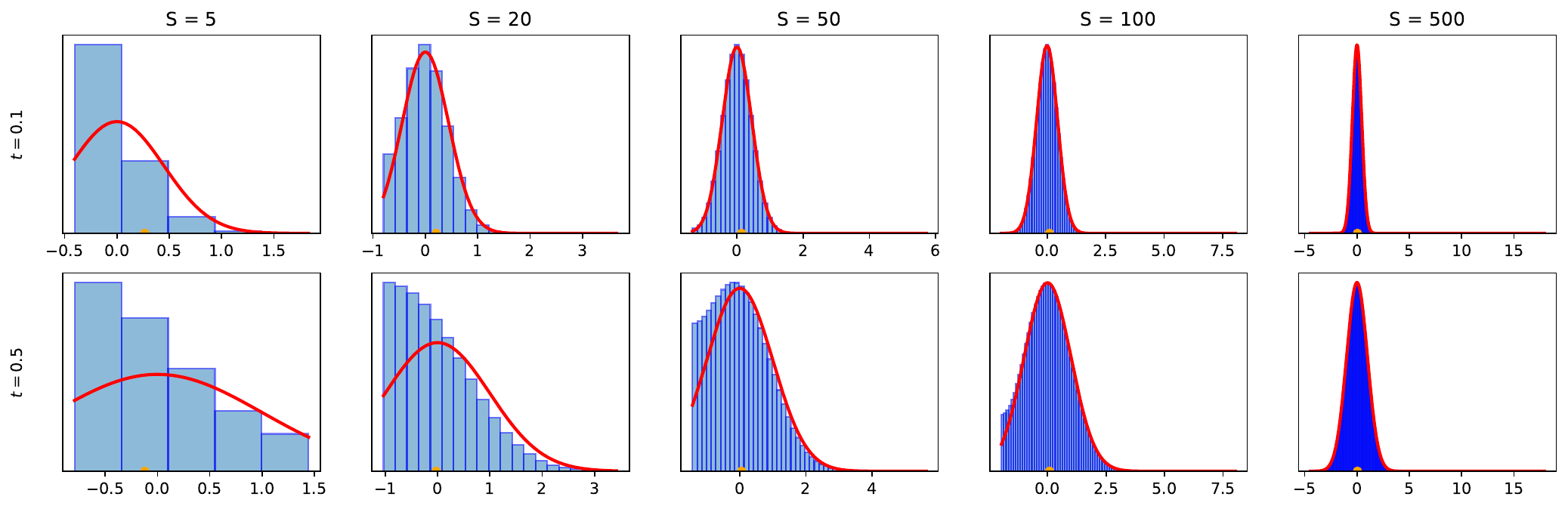}
    \caption{Scaling limit of the introduced perturbations. Top: Convergence of rescaled cyclical and ordinal perturbations $y_{S^2t}/S$ for base time parameters $t= 0.01$ and $t= 0.05$ to Gaussian $[0, 1)$ with non-trivial boundary conditions. One can see that the perturbation converges to a fixed shape on the normalised state space. Bottom: Convergence of rescaled cyclical and ordinal perturbation $(y_{St} - \mathbb E[y_{St}])/\sqrt{S}$ for base time parameters $t= 0.1$ and $t= 0.5$ to Gaussian on $\mathbb R$ (red line). The orange mark indicates the initial state. One can see that the perturbation remains non-trivial as the state space grows to infinity at rate $\sqrt{S}$.}
    \label{fig:scaling_limit}
\end{figure}
\section{Estimation of Energy Discrepancy}\label{app:EstimationED}
The energy discrepancy functional takes the form
\begin{equation}
    \ED_q (p_{\mathrm{data}}, U) := \mathbb{E}_{p_{\mathrm{data}}(\x)} [U(\x)] - \mathbb{E}_{p_{\mathrm{data}}(\x)}\mathbb{E}_{q(\y | \x)} [U_q (\y)]. 
\end{equation}
for a conditional perturbing density $q$ and contrastive potential
\begin{equation}
    U_q (\y) := - \log \sum_{\x'\in\mathcal X} q(\y | \x') \exp(-U(\x'))
\end{equation}
While in theory energy discrepancy yields a valid training functional for energy-based models for almost any choice of conditional distribution $q$, the conditional distribution also needs to allow an estimation of $U_q$ with low variance. This is particularly easy when $q$ is symmetric, i.e. $q(\mathbf y\vert\mathbf x) = q(\mathbf x\vert\mathbf y)$ in which case the contrastive potential can be expressed as an expectation which can readily be approximated from samples. This leads to
\begin{equation}\label{equ:stabilised-loss-function}
    \mathcal L_{q, M, w}(U) := \frac{1}{N} \sum_{i=1}^N \log\left(w+ \sum_{j= 1}^M \exp(U(\x^i) - U({\x}_-^{i,j}))\right) - \log(M)
\end{equation}
with $\x^i \sim p_\data$, $\y^i \sim q_t(\cdot | \x^i)$, and ${\x}_-^{i,j} \sim q_t(\cdot\vert \mathbf \y^i)$, where the offset $w$ stabilises the loss approximation as discussed in \citet{schroeder2023energy}. The interpretation of $w$ is that contributions from negative samples with $U({\x}_-^{i,j}) > U(\x^i)$ are exponentially suppressed as contributions to the loss functional, thus avoiding the energies of negative samples to explode.
\subsection{Binary Case}\label{subsec:Binary}
On binary spaces, the construction of perturbations is particularly simple. We give some details in this subsection.

\textbf{Bernoulli perturbation.} 
As proposed previously in \cite[Appendix B.3]{schroeder2023energy}, $q$ can be defined as a Bernoulli distribution. Specifically, for $\xib \sim \mathrm{Bernoulli}(\varepsilon)^d, \varepsilon \in (0, 1)$ define the Bernoulli perturbed data point as $\y = \x + \xib \,\mod 2$. This induces a symmetric transition density $q(\y|\x)$ on $\{0, 1\}^d$. The Bernoulli random variable $\xib_k$ emulates an indicator function, signifying in each dimension whether to flip the entry of $\x$. The value of $\epsilon$ controls the information loss induced by the perturbation. In theory, larger values of $\epsilon$ lead to a more data-efficient loss, while smaller values of $\epsilon$ may be more practical as they contribute to improved training stability.

It is easy to compute the matrix exponential
\begin{equation}\label{eq:MatrixExponential2D}
    \exp\left( t\begin{pmatrix}
        -1 & 1 \\
        1 & -1
    \end{pmatrix} \right)
    =
    \frac{1}{2}\begin{pmatrix}
        1+e^{-2t} & 1- e^{-2t} \\
        1- e^{-2t} & 1+ e^{-2t}
    \end{pmatrix}\xrightarrow{t\to\infty} \begin{pmatrix}
        1/2 & 1/2 \\
        1/2 & 1/2
    \end{pmatrix}
\end{equation}
thus relating the continuous time Markov chain framework on $\{0, 1\}$ to a Bernoulli perturbation with parameter $0.5*(1- e^{-2t})$.

\textbf{Neighbourhood-based perturbation and grid perturbation.}
Inspired by concrete score matching \citep{Meng2022concreteSM}, one can introduce a perturbation scheme based on neighbourhood maps: $\x \mapsto \mathcal{N}(\x)$, which assigns each data point $\x \in \mathcal{X}$ a set of neighbours $\mathcal{N}(\x)$. In this case, the forward transition density is given by the uniform distribution over the set of neighbours, {\it i.e.,} $q(\y\vert\x) = \frac{1}{|\mathcal{N}(\x)|} \delta_{\mathcal{N}(\x)}(\y)$. A special case is the grid neighbourhood
\begin{equation}
    \mathcal N_{\mathrm{grid}}(\x) = \{ \y\in \{0, 1\}^d\, :\, \y-\x = \pm \e_k,\, k = 1, 2, \dots, d\},
\end{equation}
where $\e_k$ is a vector of zeros with a one in the $k$-th entry. Notably, this neighbourhood structure also exhibits symmetry, {\it i.e.,} $\mathcal N_{\mathrm{grid}}^{-1}(\x) = \mathcal N_{\mathrm{grid}}(\x)$. The same perturbation can be derived from an Euler discretisation of the continuous time Markov chain. On a binary space we have for $t = 1$ for the same rate matrix as in \cref{eq:MatrixExponential2D}
\begin{equation}
    F := \exp(R) \approx \mathrm{id} + \begin{pmatrix}
        -1 & 1 \\
        1 & -1
    \end{pmatrix} = \begin{pmatrix}
        0 & 1 \\
        1 & 0
    \end{pmatrix}
\end{equation}
which is a completely deterministic perturbation which always changes the state of the data point. Combining this with the localisation to a grid we find
\begin{equation}
    q(\mathbf y \vert \mathbf x) = \sum_{k=1}^d \frac{1}{d} F_{yx} = \left\{ \begin{matrix}
        1/d & \y-\x = \pm \e_k \\
        0 & \text{else}
    \end{matrix}\right.
\end{equation}
thus recovering the grid perturbation.
\subsection{Connection to pseudo-likelihood estimation}
Define $q(\mathbf y \vert \mathbf x) = \frac{1}{d} \sum_{k=1}^d \delta( y_k, \Box)\delta(\mathbf y_{\neg k} , \mathbf x_{\neg k})$ which masks exactly one entry of the data vector. For simiplicity, we write $\mathbf y = \mathbf x_{\neg k}$ for the masked state and ommit the $\Box$. Then, energy discrepancy is given for a sampled perturbation $\mathbf y = \mathbf x_{\neg k}$ as
\begin{align}
    &U_\theta(\mathbf x) + \log \sum_{\mathbf x\in \mathcal X} q(\mathbf y = \mathbf x_{\neg k} \vert \mathbf x) \exp(-U_\theta(\mathbf x))\\\notag
    &= -\log \frac{\exp(-U_\theta(\mathbf x))}{\sum_{s \in \{1, 2, \dots, S_k\}}\exp(-U_\theta(x_1, \dots, x_k = s, \dots, x_d))} = -\log p_\theta(x_k \vert \mathbf x_{\neg k})
\end{align}
Hence, this specific ED loss function is indeed a Monte Carlo approximation of pseudo-likelihood. Energy discrepancy offers additional flexibility through the tunable choice of $t$ and $M$, thus making ED adaptable to the structure of the underlying space and more efficient in practice, since the normalisation of the pseudo-likelihood is only computed from $M$ samples and does not require the integration along an entire state space dimension.

\begin{figure*}[!t]
\centering
    \begin{minipage}[t]{0.49\linewidth}
    \centering
    \begin{algorithm}[H]
    \setstretch{1.275}
    \caption{Training} \small
    \label{alg:TabED-training}
    \begin{algorithmic}[1] 
        \STATE perturb the training sample $\x$ using \eqref{eq:app-perturb-dist} \\ $\y \sim q_t(\cdot \vert \x)$
        \STATE generate $M$ negative samples \\ $\x_-^i \sim q_t(\cdot \vert \y), \quad i=1,2,\dots,M$
        \STATE evaluate the energy difference \\ $d_\theta \leftarrow \frac{1}{M} \sum_{i=1}^M \exp(U_\theta (\x) - U_\theta (\x_-^i))$
        \STATE update parameter $\theta$ using \eqref{equ:stabilised-loss-function} \\ $\theta \leftarrow \theta - \eta \nabla_\theta \log (w/M + d_\theta)$
    \end{algorithmic}
    \end{algorithm}
    \end{minipage}
    \begin{minipage}[t]{0.49\linewidth}
    \centering
    \begin{algorithm}[H]
    \caption{Sampling} \small
    \label{alg:TabED-sampling} 
    \begin{algorithmic}[1] 
        \STATE initialise samples \\ $\x^\num \sim \mathcal{N}(0,\mathbf{I})$; $\x^\cat \sim \bigotimes_{k = 1}^{d_\cat} \mathrm{Uniform}(S_k)$
        \FOR{$1,\dots,N$}
        \FOR{$k \!\leftarrow\! 1,\dots, d_\cat$}
            \STATE update numerical features using \eqref{eq:app-sample-num} \\ $\x^\num \!\leftarrow\! \!\!\x^\num \!-\! \frac{\epsilon}{2} U_\theta (\x) \!+\! \sqrt{\epsilon} \boldsymbol{\omega}$, $\boldsymbol{\omega} \!\sim\! \mathcal{N}(0,\mathbf{I})$ \\
            \STATE update categorical features using \eqref{eq:app-sample-cat} \\ $x^\cat_k \leftarrow x^\cat_k \sim p_\theta (x^\cat_k | \x^\num, \x^\cat_{ \neg k})$
        \ENDFOR
        \ENDFOR
    \end{algorithmic}
    \end{algorithm}
    \end{minipage}
\vspace{-1mm}
\caption{The training and sampling procedures of energy discrepancy on tabular data. We use one training sample only to illustrate.}
\vspace{-1mm}
\end{figure*}

\section{Tabular Data Synthesising with Energy-Based Models} \label{sec:app-ebm-tabular}
In this section, we introduce how to use energy discrepancy for training an energy-based model on tabular data.
Let $d_\num$ and $d_\cat$ be the number of numerical columns and categorical columns, respectively. Each row in the table is a data point represented as a vector of numerical features and categorical features $\x = [\x^\num, \x^\cat]$, where $\x^\num \in \mathbb{R}^{d_\num}$ and $\x^\cat \in \bigotimes_{k = 1}^{d_\cat} \{1, \dots, S_k\}$. 
To train an EBM with energy discrepancy, one should define the perturbation methods, which can be done by solving the differential equation in \eqref{eq:pde-heat-diffuion}. For the numerical features, we choose the Gaussian perturbation as in \cite{schroeder2023energy}, which has the transition probability in the form of 
\begin{align}
    q^\num_t (\y^\num | \x^\num) = \mathcal{N}(\y^\num | \x^\num, t \mathbf{I}). \nonumber
\end{align}
For the categorical features, we perturb each attribution independently:
\begin{align}
    q^\cat_t (\y^\cat | \x^\cat) = \prod_{k=1}^{d_\cat} q_t (y^\cat_k | x^\cat_k). \nonumber
\end{align}
Accordingly, there are three different perturbation methods for $q_t (y^\cat_k | x^\cat_k)$:
\begin{itemize}
    \item Uniform perturbation defined in \eqref{equ:uniform_perturbation_estimation}:
    \begin{align}
        q^{\mathrm{uni}}_t(y^\cat_k \vert x^\cat_k) = (1 - S_k t)\delta(y^\cat_k, x^\cat_k) + t\sum_{k\neq x^\cat_k} \delta(y^\cat_k, x^\cat_k) \nonumber
    \end{align}
    \item Cyclical perturbation defined in \eqref{eq:cyclical_perturbation_estimation}:
    \begin{align}
        q^\mathrm{cyc}_t(y^\cat_k \vert x^\cat_k) = \frac{1}{S}\sum_{p=1}^S \exp(2\pi \mathrm i y^\cat_k\omega^\mathrm{cyc}_p)\, \exp\big((2\cos(2\pi \omega^\mathrm{cyc}_p) - 2) t\big)\,\exp(-2\pi \mathrm i x^\cat_k\omega^\mathrm{cyc}_p) \nonumber
    \end{align}
    \item Ordinal perturbation defined in \eqref{eq:ordinal_perturbation_estimation}:
    \begin{align}
        \!\!\!\!\!\!\!\!\!\!\!\!\!\!\!\!\!\!\!\!\!q^\mathrm{ord}_t(y^\cat_k \vert x^\cat_k) \!=\! \frac{2}{S} \sum_{p=1}^S \frac{1}{z_p}\, \!\!\cos((2y^\cat_k \!-\! 1)\pi \omega^\mathrm{ord}_p) \,\exp\!\big(\!(2\cos(2\pi \omega^\mathrm{ord}_p) \!-\! 2) t\big)\, \!\!\cos((2x^\cat_k \!-\! 1)\pi\omega^\mathrm{ord}_p) \nonumber
    \end{align}
\end{itemize}

To reduce the scale of noise, we further introduce grid perturbation, which involves perturbing only one attribute at a time
\begin{align}
    q^\cat_t (\y^\cat | \x^\cat) = \frac{1}{d_\cat} \sum_{k=1}^{d_\cat} q_t (y^\cat_k | x^\cat_k). \nonumber
\end{align}
Theoretically, grid perturbation can be used alongside any type of perturbation described in 
(\ref{equ:uniform_perturbation_estimation},~\ref{eq:cyclical_perturbation_estimation},~\ref{eq:ordinal_perturbation_estimation}).
In our experimental studies, we only explore the combination of grid perturbation with uniform perturbation. By combining the Gaussian perturbation and categorical perturbation together, we can then draw the negative samples $\x_-$ via $\y \sim q_t (\cdot | \x)$, and $\x_- \sim q_t(\cdot | \y)$, where
\begin{align} \label{eq:app-perturb-dist}
    q_t (\y | \x) =  q^\num_{t_\num} (\y^\num | \x^\num) q^\cat_{t_\cat} (\y^\cat | \x^\cat).
\end{align}
Therefore, the energy function $U_\theta$ can be learned by minimising the loss function in \eqref{equ:stabilised-loss-function}. We summarise the training procedure in \cref{alg:TabED-training}.

After training, new tabular data is synthesised by alternately applying Langevin dynamics and Gibbs sampling. Specifically, we update the numerical feature $\x^\num$ via Langevin dynamics
\begin{align} \label{eq:app-sample-num}
    \x^\num \leftarrow \x^\num - \frac{\epsilon}{2} \nabla_{\x^\num} U_\theta ([\x^\num, \x^\cat]) + \sqrt{\epsilon} \boldsymbol{\omega}, \quad \boldsymbol{\omega} \sim \mathcal{N}(0, \mathbf{I})
\end{align}
For the categorical feature $\x^\cat$, we employ Gibbs sampler
\begin{align} \label{eq:app-sample-cat}
    x^\cat_k \sim p_\theta (x^\cat_k | \x^\num, \x^\cat_{ \neg k}) \propto \exp (-U_\theta ([\x^\num, x^\cat_k, \x^\cat_{ \neg k}])), \quad k=1,2,\dots,d_\cat,
\end{align}
where $\x^\cat_{ \neg k}$ denotes the vector $\x^\cat$ excluding the $k$-th attribute.
The sampling procedure is summarised in \cref{alg:TabED-sampling}.

\begin{figure}[!t]
    \centering
    \begin{minipage}[t]{0.05\linewidth}
        \centering
        \raisebox{-.5\height}{\includegraphics[width=.48in]{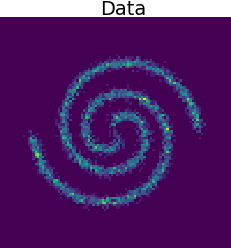}}
    \end{minipage}
    \begin{minipage}[t]{0.25\linewidth}
        \centering
        \includegraphics[width=.48in]{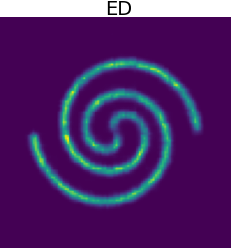}
        \includegraphics[width=.48in]{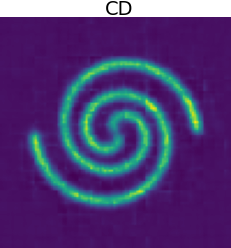}
        \includegraphics[width=.48in]{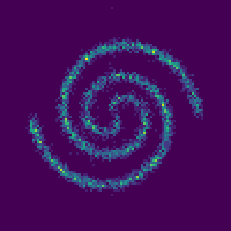}
        \includegraphics[width=.48in]{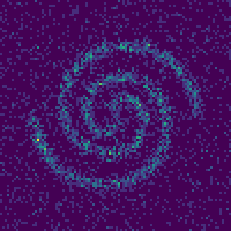}
    \end{minipage}
    \begin{minipage}[t]{0.05\linewidth}
        \centering
        \raisebox{-.5\height}{\includegraphics[width=.48in]{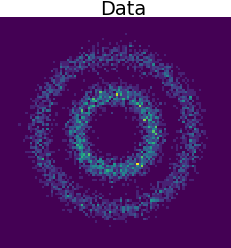}}
    \end{minipage}
    \begin{minipage}[t]{0.25\linewidth}
        \centering
        \includegraphics[width=.48in]{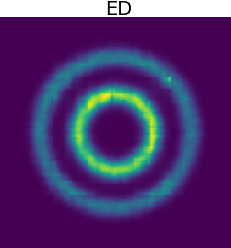}
        \includegraphics[width=.48in]{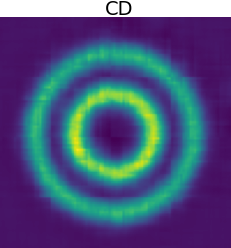}
        \includegraphics[width=.48in]{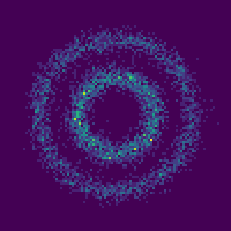}
        \includegraphics[width=.48in]{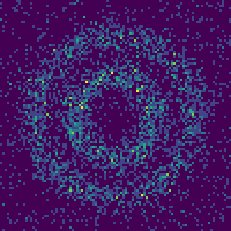}
    \end{minipage}
    \begin{minipage}[t]{0.05\linewidth}
        \centering
        \raisebox{-.5\height}{\includegraphics[width=.48in]{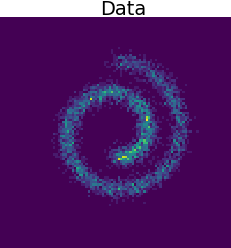}}
    \end{minipage}
    \begin{minipage}[t]{0.25\linewidth}
        \centering
        \includegraphics[width=.48in]{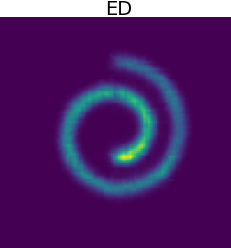}
        \includegraphics[width=.48in]{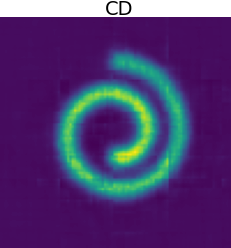}
        \includegraphics[width=.48in]{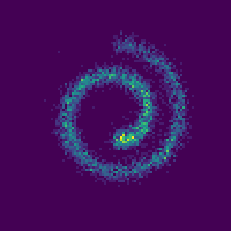}
        \includegraphics[width=.48in]{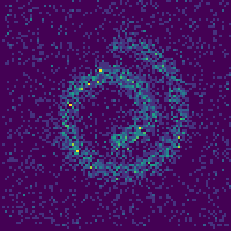}
    \end{minipage}\\
    \begin{minipage}[t]{0.05\linewidth}
        \centering
        \raisebox{-.5\height}{\includegraphics[width=.48in]{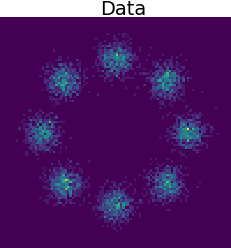}}
    \end{minipage}
    \begin{minipage}[t]{0.25\linewidth}
        \centering
        \includegraphics[width=.48in]{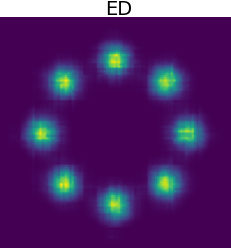}
        \includegraphics[width=.48in]{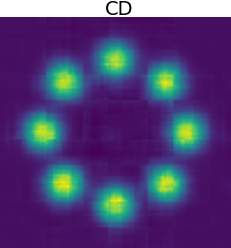}
        \includegraphics[width=.48in]{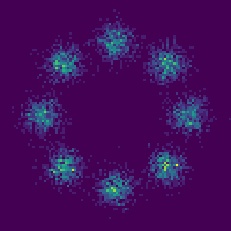}
        \includegraphics[width=.48in]{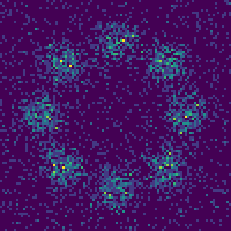}
    \end{minipage}
    \begin{minipage}[t]{0.05\linewidth}
        \centering
        \raisebox{-.5\height}{\includegraphics[width=.48in]{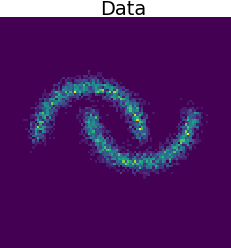}}
    \end{minipage}
    \begin{minipage}[t]{0.25\linewidth}
        \centering
        \includegraphics[width=.48in]{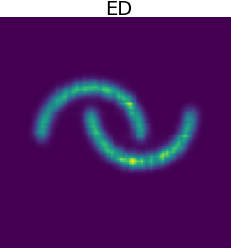}
        \includegraphics[width=.48in]{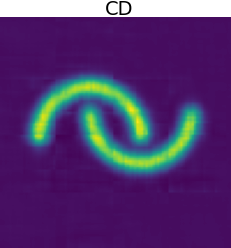}
        \includegraphics[width=.48in]{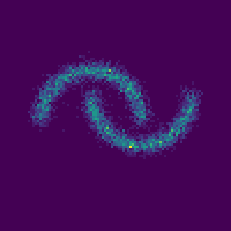}
        \includegraphics[width=.48in]{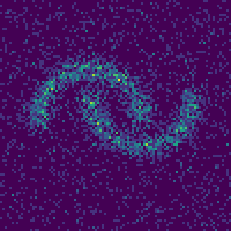}
    \end{minipage}
    \begin{minipage}[t]{0.05\linewidth}
        \centering
        \raisebox{-.5\height}{\includegraphics[width=.48in]{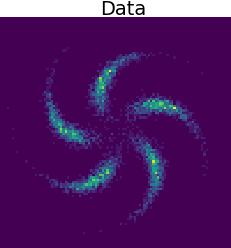}}
    \end{minipage}
    \begin{minipage}[t]{0.25\linewidth}
        \centering
        \includegraphics[width=.48in]{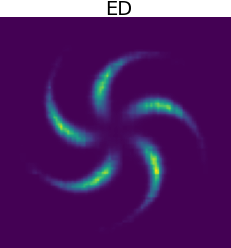}
        \includegraphics[width=.48in]{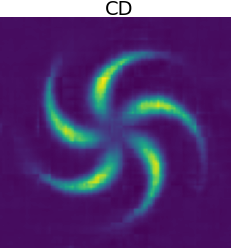}
        \includegraphics[width=.48in]{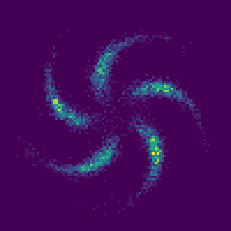}
        \includegraphics[width=.48in]{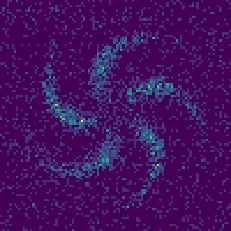}
    \end{minipage}
    \caption{Additional density estimation results on the dataset with $16$ dimensions and $5$ states.}
    \label{fig:appendix-exp-synthetic-5states}
\end{figure}

\begin{figure}[!t]
    \centering
    \begin{minipage}[t]{0.05\linewidth}
        \centering
        \raisebox{-.5\height}{\includegraphics[width=.48in]{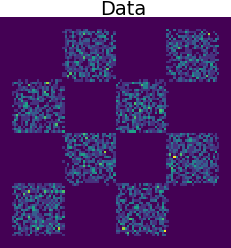}}
    \end{minipage}
    \begin{minipage}[t]{0.25\linewidth}
        \centering
        \includegraphics[width=.48in]{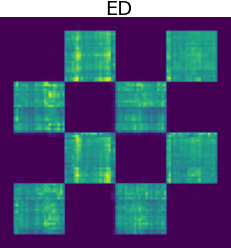}
        \includegraphics[width=.48in]{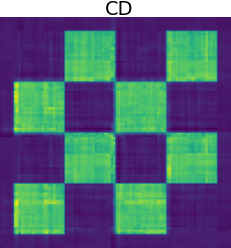}
        \includegraphics[width=.48in]{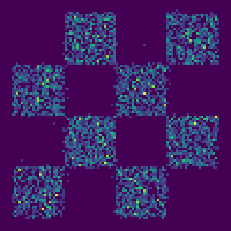}
        \includegraphics[width=.48in]{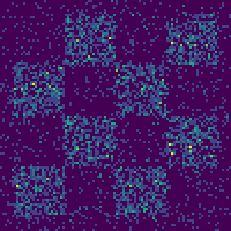}
    \end{minipage}
    \begin{minipage}[t]{0.05\linewidth}
        \centering
        \raisebox{-.5\height}{\includegraphics[width=.48in]{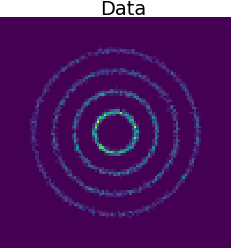}}
    \end{minipage}
    \begin{minipage}[t]{0.25\linewidth}
        \centering
        \includegraphics[width=.48in]{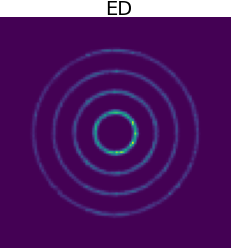}
        \includegraphics[width=.48in]{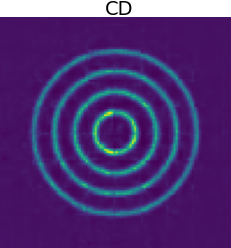}
        \includegraphics[width=.48in]{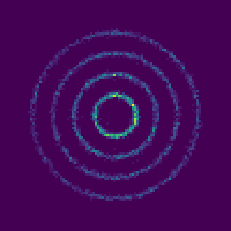}
        \includegraphics[width=.48in]{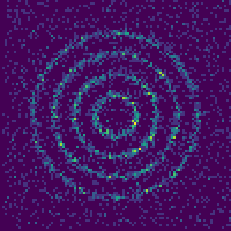}
    \end{minipage}
    \begin{minipage}[t]{0.05\linewidth}
        \centering
        \raisebox{-.5\height}{\includegraphics[width=.48in]{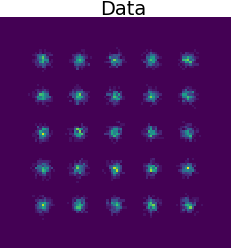}}
    \end{minipage}
    \begin{minipage}[t]{0.25\linewidth}
        \centering
        \includegraphics[width=.48in]{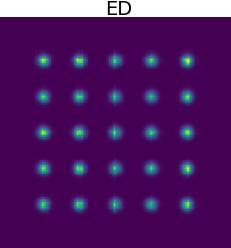}
        \includegraphics[width=.48in]{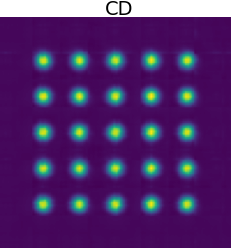}
        \includegraphics[width=.48in]{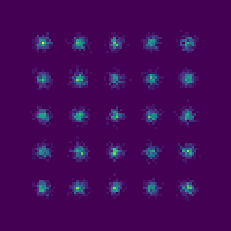}
        \includegraphics[width=.48in]{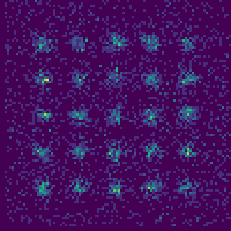}
    \end{minipage}\\
    \begin{minipage}[t]{0.05\linewidth}
        \centering
        \raisebox{-.5\height}{\includegraphics[width=.48in]{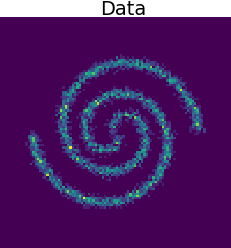}}
    \end{minipage}
    \begin{minipage}[t]{0.25\linewidth}
        \centering
        \includegraphics[width=.48in]{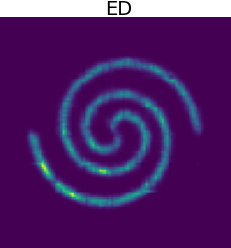}
        \includegraphics[width=.48in]{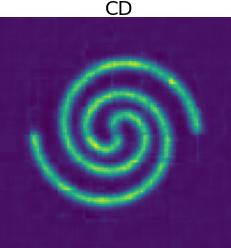}
        \includegraphics[width=.48in]{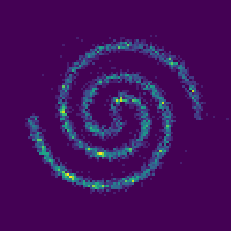}
        \includegraphics[width=.48in]{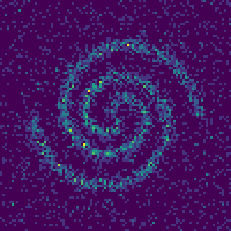}
    \end{minipage}
    \begin{minipage}[t]{0.05\linewidth}
        \centering
        \raisebox{-.5\height}{\includegraphics[width=.48in]{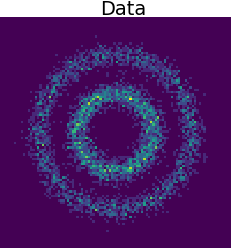}}
    \end{minipage}
    \begin{minipage}[t]{0.25\linewidth}
        \centering
        \includegraphics[width=.48in]{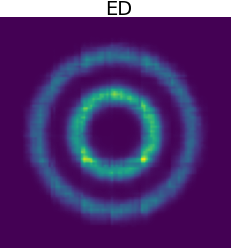}
        \includegraphics[width=.48in]{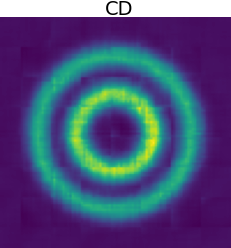}
        \includegraphics[width=.48in]{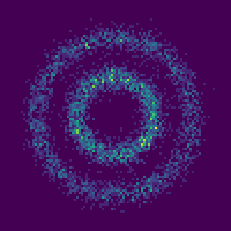}
        \includegraphics[width=.48in]{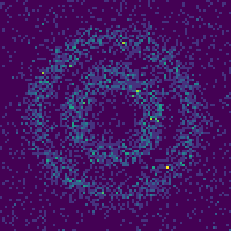}
    \end{minipage}
    \begin{minipage}[t]{0.05\linewidth}
        \centering
        \raisebox{-.5\height}{\includegraphics[width=.48in]{figures/density_estimation/vc5_swissroll_data.png}}
    \end{minipage}
    \begin{minipage}[t]{0.25\linewidth}
        \centering
        \includegraphics[width=.48in]{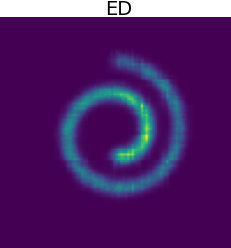}
        \includegraphics[width=.48in]{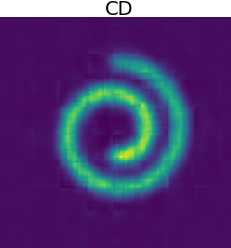}
        \includegraphics[width=.48in]{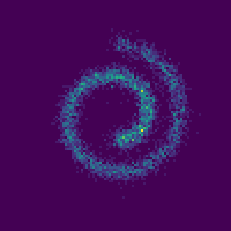}
        \includegraphics[width=.48in]{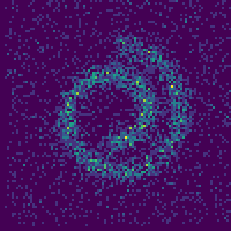}
    \end{minipage}\\
    \begin{minipage}[t]{0.05\linewidth}
        \centering
        \raisebox{-.5\height}{\includegraphics[width=.48in]{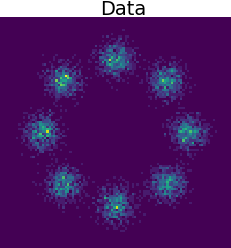}}
    \end{minipage}
    \begin{minipage}[t]{0.25\linewidth}
        \centering
        \includegraphics[width=.48in]{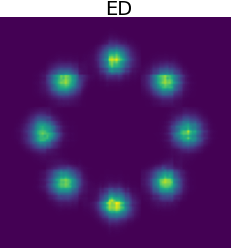}
        \includegraphics[width=.48in]{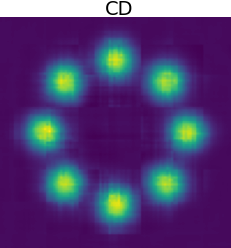}
        \includegraphics[width=.48in]{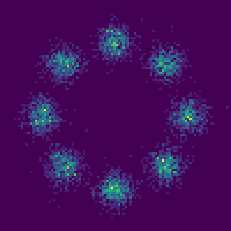}
        \includegraphics[width=.48in]{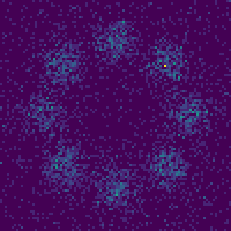}
    \end{minipage}
    \begin{minipage}[t]{0.05\linewidth}
        \centering
        \raisebox{-.5\height}{\includegraphics[width=.48in]{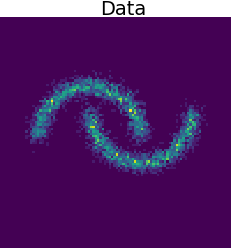}}
    \end{minipage}
    \begin{minipage}[t]{0.25\linewidth}
        \centering
        \includegraphics[width=.48in]{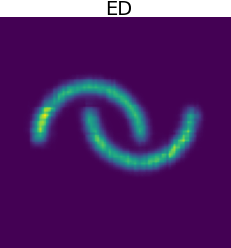}
        \includegraphics[width=.48in]{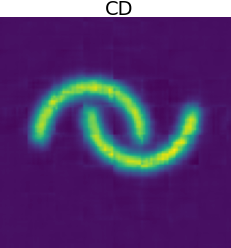}
        \includegraphics[width=.48in]{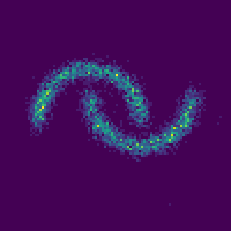}
        \includegraphics[width=.48in]{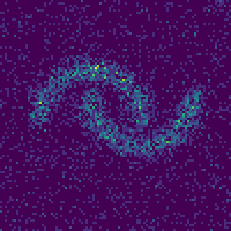}
    \end{minipage}
    \begin{minipage}[t]{0.05\linewidth}
        \centering
        \raisebox{-.5\height}{\includegraphics[width=.48in]{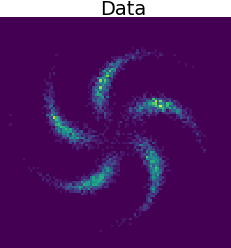}}
    \end{minipage}
    \begin{minipage}[t]{0.25\linewidth}
        \centering
        \includegraphics[width=.48in]{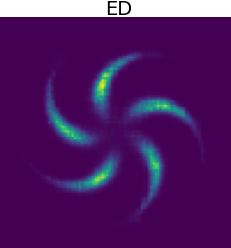}
        \includegraphics[width=.48in]{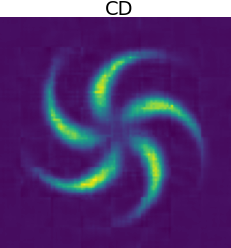}
        \includegraphics[width=.48in]{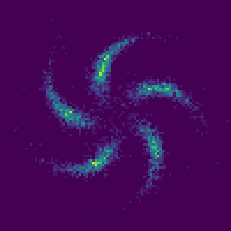}
        \includegraphics[width=.48in]{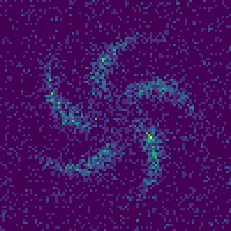}
    \end{minipage}
    \caption{Additional density estimation results on the dataset with $12$ dimensions and $10$ sates.}
    \label{fig:appendix-exp-synthetic-10states}
\end{figure}

\begin{table*}[!t]
\small
\setlength{\tabcolsep}{1.5mm}
\centering
\caption{
Experimental results of discrete density estimation. We display the negative log-likelihood (NLL). The results of baselines are taken from \cite{zhang2022generative}.
}
\label{tab:synthetic_nll}
\begin{tabular}{c|l|ccccccc}
\toprule
Metric & Method & 2spirals & 8gaussians & circles & moons & pinwheel & swissroll & checkerboard \\
\midrule
\multirow{6}{*}{NLL$\downarrow$} 
& PCD    &  $20.094$ & $19.991$ &$20.565$&$19.763$&$19.593$&$20.172$&$21.214$ \\
& ALOE+  &  $20.062$ & $19.984$ & $20.570$ & $19.743$ & $19.576$ & $20.170$ & $21.142$\\
& EB-GFN & ${20.050}$ & ${19.982}$ & $\textbf{20.546}$ & ${19.732}$ & $\textbf{19.554}$ & ${20.146}$ & ${20.696}$ \\
\arrayrulecolor{black!40}\cmidrule{2-9}
& \EDB & $\textbf{20.039}$ & ${19.992}$ & ${20.601}$ & $\textbf{19.710}$ & ${19.568}$ & $\textbf{20.084}$ & ${20.679}$ \\
& \EDG & ${20.049}$ & $\textbf{19.965}$ & ${20.601}$ & ${19.715}$ & ${19.564}$ & ${20.088}$ & $\textbf{20.678}$ \\
\bottomrule
\end{tabular}
\end{table*}

\begin{table*}[!t]
\small
\setlength{\tabcolsep}{1.5mm}
\centering
\caption{
Experimental results of discrete density estimation.
We display the MMD (in units of $1\times 10^{-4}$). The results of baselines are taken from \cite{zhang2022generative}.
}
\label{tab:synthetic_mmd}
\begin{tabular}{c|l|ccccccc}
\toprule
Metric & Method & 2spirals & 8gaussians & circles & moons & pinwheel & swissroll & checkerboard \\
\midrule
\multirow{6}{*}{MMD$\downarrow$} 
& PCD    &  $2.160$&$0.954$&$0.188$&$0.962$&$0.505$&$1.382$& $2.831$ \\
& ALOE+  &  $ {0.149} $ & ${0.078}$ & $0.636$ & $0.516$ & $1.746$ & $0.718$ & $12.138$ \\
& EB-GFN &  $0.583$ & $0.531$ & ${0.305}$ & ${0.121}$ & ${0.492}$ & ${0.274}$ & $\textbf{1.206}$\\ 
\arrayrulecolor{black!40}\cmidrule{2-9}
& \EDB & ${0.120}$ & ${0.014}$ & ${0.137}$ & $\textbf{-0.088}$ & $\textbf{0.046}$ & $\textbf{0.045}$ & ${1.541}$ \\
& \EDG & $\textbf{0.097}$ & $\textbf{-0.066}$ & $\textbf{0.022}$ & ${0.018}$ & ${0.351}$ & ${0.097}$ & ${2.049}$ \\
\bottomrule
\end{tabular}
\end{table*}

\section{Additional Experimental Results} \label{sec-appendix-experiments}
In this section, we present the detailed experimental settings and additional results. All experiments are conducted on a single Nvidia RTX 3090 GPU with 24GB of VRAM.

\subsection{Discrete Density Estimation} \label{appendix-sec-discrete-density-estimation}

\textbf{Experimental Details.}
This experiment keeps a consistent setting with \cite{dai2020learning}. We first generate 2D floating-points from a continuous distribution $\hat{p}$ which lacks a closed form but can be easily sampled. Then, each sample $\hat{\x} := [\hat{\x}_1, \hat{\x}_2] \in \mathbb{R}^2$ is converted to a discrete data point $\x \in \{0,1\}^{32}$ using Gray code. To be specific, given $\hat{\x} \sim \hat{p}$, we quantise both $\hat{\x}_1$ and $\hat{\x}_2$ into $16$-bits binary representations via Gray code \citep{gray1953pulse}, and concatenate them together to obtain a $32$-bits vector $\x$. As a result, the probabilistic mass function in the discrete space is $p(\x) \propto \hat{p} \left( \left[ \operatorname{GrayToFloat}(\x_{1:16}), \operatorname{GrayToFloat}(\x_{17:32}) \right] \right)$. 
To extend datasets beyond binary cases, we adhere to the same protocol but utilise base transformation instead. 
This transformation enables the conversion of floating-point coordinates into discrete variables with different state sizes.
It is important to highlight that learning EBMs in such discrete spaces presents challenges due to the highly non-linear characteristics of both the Gray code and base transformation.

We parameterise the energy function using a $4$ layer MLP with $256$ hidden dimensions and Swish \citep{ramachandran2017searching} activation. To train the EBM, we adopt the Adam optimiser with a learning rate of $0.0001$ and a batch size of $128$ to update the parameter for $10^5$ steps. For the energy discrepancy, we choose $w=1, M=32$ and the grid perturbation for all variants. For contrastive divergence, we employ short-run MCMC using Gibbs sampling with $10$ rounds (i.e., $10 * S$ steps).

After training, we quantitatively evaluate all methods using the negative log-likelihood (NLL) and the maximum mean discrepancy (MMD). To be specific, the NLL metric is computed based on $4,000$ samples drawn from the data distribution, and the normalisation constant is estimated using importance sampling with $1,000,000$ samples drawn from a variational Bernoulli distribution with $p=0.5$. For the MMD metric, we follow the setting in \cite{zhang2022generative}, which adopts the exponential Hamming kernel with $0.1$ bandwidth. Moreover, the reported performances are averaged over 10 repeated estimations, each with $4,000$ samples, which are drawn from the learned energy function via Gibbs sampling.

\textbf{Qualitative Results.}
In \cref{fig:appendix-exp-synthetic-5states,fig:appendix-exp-synthetic-10states}, we present additional qualitative results of the learned energy on datasets with $5$ and $10$ states. We see that ED consistently yields more accurate energy landscapes compared to CD. 
Notably, we only showcase results using grid perturbation with the uniform rate matrix, as qualitative findings are consistent across different perturbation methods.
Additionally, we empirically observe that gradient-based Gibbs sampling methods \citep{grathwohl2021oops,zhang2022langevin} tend to generate samples outside the data support more readily. In this regard, we only display the results of CD methods using vanilla Gibbs sampling.

\textbf{Quantitative Results.}
The quantitative results are illustrated in \cref{tab:synthetic_nll,tab:synthetic_mmd}, indicating the superior performance of our approaches in most scenarios. Notably, previous studies on discrete EBM modelling exclusively focus on binary cases. As a result, we only present the quantitative comparison for the dataset with 2 states.

\begin{table}[!t] 
    \centering
    \caption{Statistics of the real-world datasets.} 
    \label{tbl:appendix-dataset-stat}
    \small
    \begin{threeparttable}
    \resizebox{\textwidth}{!}{
	\begin{tabular}{lcccccccc}
            \toprule
            \textbf{Dataset} & \# Rows  & \# Num & \# Cat & \# Train & \# Validation & \# Test & Task Type  \\
            \midrule 
            \textbf{\href{https://archive.ics.uci.edu/dataset/2/adult}{Adult}} & $48,842$ & $6$ & $9$ & $28,943$ & $3,618$ & $16,281$ & Binary Classification  \\
            \textbf{\href{https://archive.ics.uci.edu/dataset/222/bank+marketing}{Bank}} & $45,211$ & $7$ & $10$ & $36,168$  & $4,521$ & $4,522$ & Binary Classification \\
            \textbf{\href{https://www.kaggle.com/datasets/sulianova/cardiovascular-disease-dataset}{Cardio}} & $70,000$ & $5$ & $6$ & $56,000$  & $7,000$ & $7,000$ & Binary Classification \\
            \textbf{\href{https://www.kaggle.com/datasets/blastchar/telco-customer-churn}{Churn}} & $7,043$ & $3$ & $15$ & $5,634$  & $704$ & $705$ & Binary Classification \\
            \textbf{\href{https://archive.ics.uci.edu/dataset/848/secondary+mushroom+dataset}{Mushroom}} & $61,096$ & $3$ & $17$ & $48,855$  & $6,107$ & $6,107$ & Binary Classification \\
            \textbf{\href{https://archive.ics.uci.edu/dataset/381/beijing+pm2+5+data}{Beijing}} & $43,824$ & $7$ & $5$ & $35,058$  & $4,383$ & $4,383$ &  Regression   \\
		\bottomrule
		\end{tabular}
    }        
  \end{threeparttable}
\end{table}

\begin{figure}[!t]
    \centering
    {
        \begin{minipage}{0.235\textwidth}
            \centering
            \includegraphics[width=1.\textwidth]{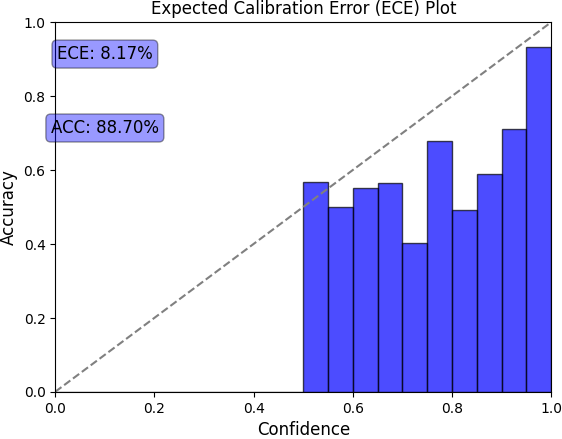}
        \end{minipage}
        \begin{minipage}{0.235\textwidth}
            \centering
            \includegraphics[width=1.\textwidth]{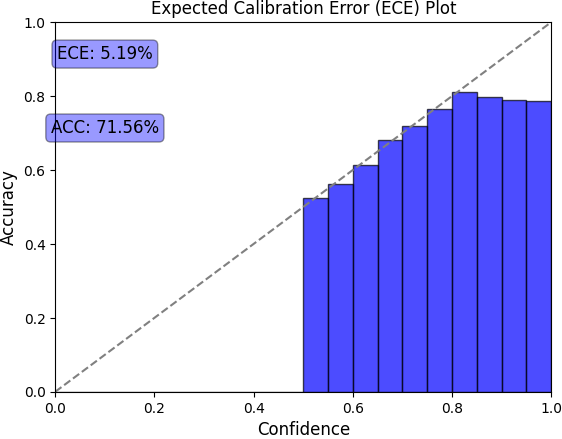}
        \end{minipage}
        \begin{minipage}{0.235\textwidth}
            \centering
            \includegraphics[width=1.\textwidth]{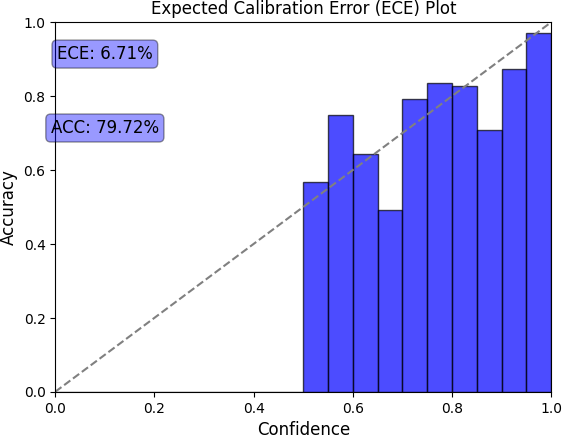}
        \end{minipage}
        \begin{minipage}{0.235\textwidth}
            \centering
            \includegraphics[width=1.\textwidth]{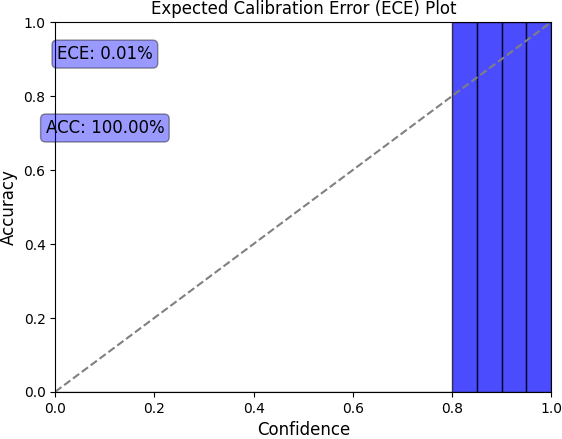}
        \end{minipage}
    }
    {
        \begin{minipage}{0.235\textwidth}
            \centering
            \includegraphics[width=1.\textwidth]{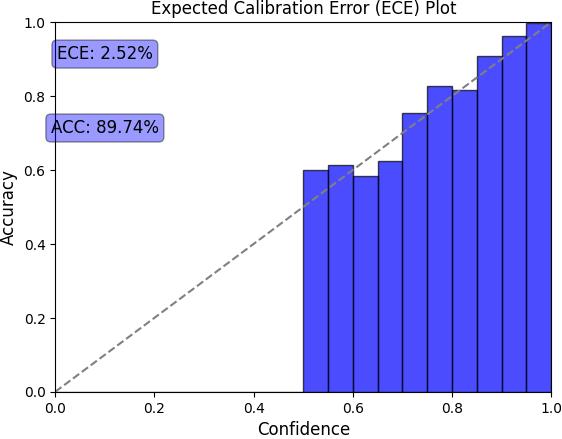}
        \end{minipage}
        \begin{minipage}{0.235\textwidth}
            \centering
            \includegraphics[width=1.\textwidth]{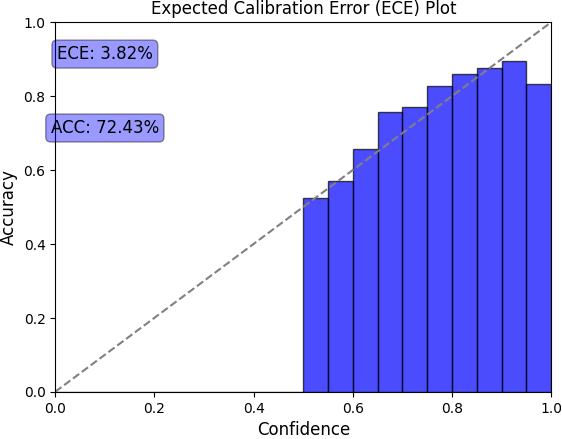}
        \end{minipage}
        \begin{minipage}{0.235\textwidth}
            \centering
            \includegraphics[width=1.\textwidth]{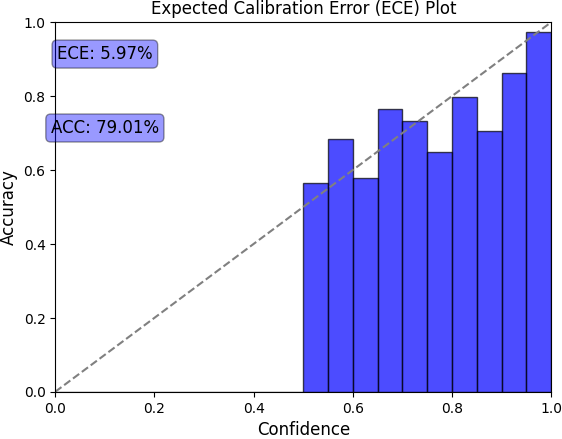}
        \end{minipage}
        \begin{minipage}{0.235\textwidth}
            \centering
            \includegraphics[width=1.\textwidth]{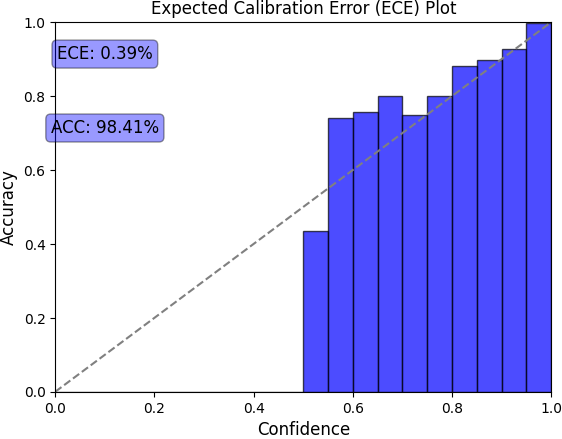}
        \end{minipage}
    }
    \caption{Comparison of calibration results between the baseline (top) and energy discrepancy (bottom) on varying datasets. Left to right: Bank, Cardio, Churn, Mushroom.}
    \label{fig:calibration-appendix}
\end{figure}

\subsection{Tabular Data Synthesising}  \label{appendix-sec-tabular}

\textbf{Experimental Details for the Synthetic Dataset.}
For the synthetic dataset, we parametrise the energy function using three MLP layers with $256$ hidden states and Swish activation.
To handle mixed data types, we transform each categorical feature into a $4$-dimensional embedding using a linear layer, and then concatenate these embeddings with the numerical features as input.
To train the EBM, we apply the Adam optimiser with a learning rate of $0.0001$ and a batch size of $128$. We update the parameters over $1,000$ epochs, with each epoch consisting of $100$ update iterations. For ED, we set $w=1, M=32$, using Gaussian perturbation for the numerical features and grid perturbation for the categorical features. For CD, we incorporate the replay buffer strategy and employ Langevin dynamics and Gibbs sampling with $50$ rounds (totalling $50 * S$ steps) for the numerical and categorical features, respectively.

\textbf{Experimental Details for the Real-world Dataset.}
\cref{tbl:appendix-dataset-stat} summarises the statistical properties of the datasets.
To parameterise the energy function and handle mixed data types, we use the same approach but with $1024$ hidden units instead of $256$. We train the model using the AdamW optimiser \citep{loshchilov2017decoupled} with a learning rate of $0.0001$ and a weight decay rate of $0.0005$. The model is trained for $20,000$ update steps with a batch size of $4096$.
For ED, Gaussian perturbation is employed for numerical features, while categorical features undergo different perturbation methods.
Specifically, TabED-Uni and TabED-Grid use uniform and grid perturbations with $t=0.1$, respectively. For TabED-Cyc and TabED-Ord, corresponding to cyclical and ordinal perturbations, quadratic scaling is applied with $t$ chosen from the best performance in $\{0.01, 0.005, 0.001\}$. Moreover, CD utilises the same algorithm as in the synthetic dataset, but with $10$ steps for short-run MCMC.
The reported results are averaged over $10$ randomly sampled synthetic data.

\textbf{Experimental Details for Calibration.}
Let $y$ and $\x$ be the target label and the rest features in the tabular data, we can transform a learned EBM $U_\theta (\x, y)$ into a deterministic classifier: $p_{\text{EBM}}(y|\x) \propto \exp(-U_\theta (\x, y))$. As a baseline for comparison, we additionally train a classifier $p_{\text{CLF}}(y|\x)$ with the same architecture by maximising the conditional likelihood: $\mathbb{E}_{p_\data} [\log p_{\text{CLF}}(y|\x)]$. In particular, we utilise the Adam optimiser with a learning rate of $0.001$ and a batch size of $4096$ to train the classifier $p_{\text{CLF}}$. The model undergoes training for $50$ epochs.

\textbf{Additional Results for Calibration.}
\cref{fig:calibration-appendix} presents additional calibration results across different datasets. It shows that the energy-based classifier learned by energy discrepancy exhibits superior calibration compared to the deterministic classifier, except for the Mushroom dataset, where the deterministic classifier achieves $100\%$ accuracy, resulting in low calibration error.

\textbf{Additional Results with Other Metrics.}
We evaluate our methods against baselines using two additional metrics: single-column density similarity and pair-wise correlation similarity. These metrics assess the similarity in the empirical distribution of individual columns and the correlations between pairs of columns in the generated versus real tables. Both metrics can be computed using the open-source \href{https://docs.sdv.dev/sdmetrics/reports/quality-report/single-table-api}{SDMetrics} API. As shown in \cref{tbl:exp-density-similarity}, the result shows that the proposed ED-based approaches either outperform or achieve comparable performance to the baselines across most datasets.

\begin{table}[!t] 
    \centering
    \caption{Results on density similarity between the synthesis and real tabular data.}
    \label{tbl:exp-density-similarity}
    \small
    \setlength{\tabcolsep}{1.mm}{
    \scalebox{.9}{
    \begin{threeparttable}
    {
        \begin{tabular}{lcccccccccccc}
            \toprule[0.8pt]
            \multirow{2}{*}{Methods} & \multicolumn{6}{c}{\textbf{Single-column Density Similarity} $\uparrow$ } & \multicolumn{6}{c}{\textbf{Pair-wise Correlation Similarity} $\uparrow$ } \\
            \cmidrule(lr){2-7}  \cmidrule(lr){8-13}
             & {{Adult}} &{{Bank}} & {Cardio} & {{Churn}} &   {{Mushroom}} & {{Beijing}} & {{Adult}} &{{Bank}} & {Cardio} & {{Churn}} &   {{Mushroom}} & {{Beijing}} \\
            \midrule 
            CTGAN & $.814$ & $.866$ & $.906$ & $.901$ & $.951$ & $.799$ & $.744$ & $.769$ & $.717$ & $.826$ & $.828$ & $.761$ \\
            TVAE  & $.783$ & $.824$ & $.892$ & $.899$ & $.965$ & $.711$ & $.669$ & $.772$ & $.687$ & $.808$ & $.919$ & $.618$ \\
            TabCD & $.719$ & $.790$ & $.824$ & $.845$ & $.618$ & $.799$ & $.522$ & $.600$ & $.629$ & $.710$ & $.428$ & $.761$ \\
            TabDDPM & $.988$ & $.998$ & $.992$ & $.976$ & $.987$ & $.980$ & $.975$ & $.894$ & $.870$ & $.953$ & $.962$ & $.946$ \\
            \midrule
            TabED-Uni & $.785$ & $.779$ & $.914$ & $.886$ & $.878$ & $.933$ & $.653$ & $.683$ & $.783$ & $.808$ & $.770$ & $.793$ \\
            TabED-Grid & $.751$ & $.766$ & $.945$ & $.846$ & $.951$ & $.951$ & $.583$ & $.768$ & $.829$ & $.764$ & $.842$ & $.842$ \\
            TabED-Cyc & $.778$ & $.826$ & $.937$ & $.834$ & $.969$ & $.751$ & $.636$ & $.703$ & $.810$ & $.755$ & $.860$ & $.685$ \\
            TabED-Ord & $.828$ & $.894$ & $.933$ & $.887$ & $.943$ & $.747$ & $.702$ & $.796$ & $.811$ & $.791$ & $.826$ & $.662$ \\
            TabED-Str & $.77$ & $.798$ & $-$ & $-$ & $-$ & $.892$ & $.632$ & $.660$ & $-$ & $-$ & $-$ & $.759$ \\
		\bottomrule[1.0pt] 
		\end{tabular}
    }
    \end{threeparttable}
    }
    }
\end{table}

\subsection{Discrete Image Modelling} \label{appendix-sec-discrete-image-modelling}

\begin{figure}[!t]
\centering
	\begin{tabular}{ccc}		
	   \includegraphics[width=0.3\textwidth]{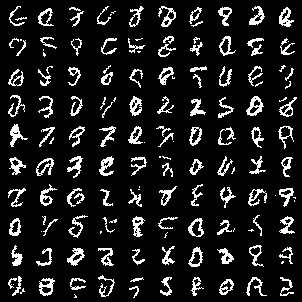}&
        \includegraphics[width=0.3\textwidth]{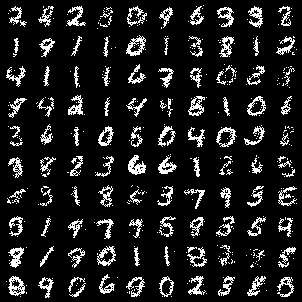}&
        \includegraphics[width=0.3\textwidth]{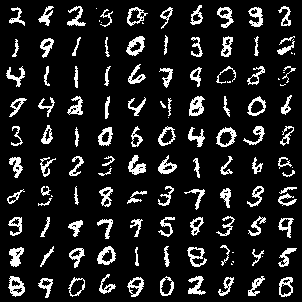}
	\end{tabular}
	\caption{Dynamic MNIST samples generated from the learned energy function using \EDG with various sampling methods. Left to right: GWG, GFlowNet, GFlowNet+GWG.}
	\label{fig:sample-ebm-appendix}
\end{figure} 

\textbf{Experimental Details.}
In this experiment, we parametrise the energy function using ResNet \citep{he2016deep} following the settings in \cite{grathwohl2021oops,zhang2022langevin}, where the network has $8$ residual blocks with $64$ feature maps. Each residual block has $2$ convolutional layers and uses Swish activation function \citep{ramachandran2017searching}. We choose $M=32, w=1$ for all variants of energy discrepancy, $\epsilon=0.001$ in Bernoulli perturbations. Note that here we choose a relatively small $\epsilon$ since we empirically find that the loss of energy discrepancy converges to a constant rapidly with larger $\epsilon$, which can not provide meaningful gradient information to update the parameters. All models are trained with Adam optimiser with a learning rate of $0.0001$ and a batch size of $100$ for $50,000$ iterations. We perform model evaluation every $5,000$ iteration by conducting Annealed Importance Sampling (AIS) with the GWG \citep{grathwohl2021oops} sampler for $10,000$ steps. The reported results are obtained from the model that achieves the best performance on the validation set. After training, we finally report the negative log-likelihood by running $300,000$ iterations of AIS. 

\textbf{Qualitative Results.}
To qualitatively assess the validity of the learned EBM, this study presents generated samples from the dynamic MNIST dataset. We first train an EBM using \EDG and then synthesise samples by employing various sampling methods, including: i) GWG \citep{grathwohl2021oops} with $1000$ steps; ii) GFlowNet with the same architecture and training procedure as per \cite{zhang2022generative}; and iii) GFlowNet followed by GWG with $100$ steps.

Empirically, we find that the quality of generated samples can be improved with more advanced sampling approaches. As depicted in \cref{fig:sample-ebm-appendix}, the GWG sampler suffers from mode collapse, leading to samples with similar patterns. In other hands, GFlowNet enhances the quality to some extent, but it produces noisy images. To address this issue, we apply GWG with $100$ steps following the GFlowNet. It can be seen that the resulting GFlowNet+GWG sampler yields the highest quality with clear digits.
These observations validate the capability of our energy discrepancies to accurately learn the energy landscape from high-dimensional datasets. 
We leave the development of a more advanced sampler in future work to further improve the quality of generated images using our energy discrepancy approaches.

\textbf{Time Complexity Comparison for Energy Discrepancy and Contrastive Divergence.}
Energy discrepancy offers greater training efficiency than contrastive divergence, as it does not rely on MCMC sampling.
In this experiment, we evaluate the running time per iteration and epoch for energy discrepancy and contrastive divergence in training a discrete EBM on the static MNIST dataset.
The experiments include contrastive divergence with varying MCMC steps and variants of energy discrepancy with a fixed value of $M=32$. The results, presented in \cref{tab:image_time_complexity}, highlight that \EDB and \EDG are the fastest options, as they do not involve gradient computations during training.

\textbf{Comparison to Contrastive Divergence with Different MCMC Steps.}
Considering the greater training efficiency of energy discrepancy over contrastive divergence, this study comprehensively compares these two methods with varying MCMC steps in contrastive divergence.
Specifically, we utilise the officially open-sourced implementation\footnote{\url{https://github.com/ruqizhang/discrete-langevin}} of DULA to conduct contrastive divergence training. 
As depicted in Table \ref{tab:image_gen_cd-n}, we find that energy discrepancy significantly outperforms contrastive divergence when employing a single MCMC step, and achieves performance comparable to CD-10. 
We attribute this superiority to the fact that CD-1 involves a biased estimation of the log-likelihood gradient due to inherent issues with non-convergent MCMC processes. In contrast, energy discrepancy does not suffer from this issue due to its consistent approximation.

\begin{table}[t]
\small
\setlength{\tabcolsep}{2.1mm}
\centering
\caption{Running time complexity comparison for energy discrepancy and contrastive divergence.}
\label{tab:image_time_complexity}
\begin{tabular}{l|ccccc}
\toprule 
Time $\backslash$ Method & CD-1 & CD-5 & CD-10 & \EDB & \EDG \\
\midrule 
Per Iteration (s) & $0.0583$ & $0.1904$ & $0.3351$ & $0.0905$ & $0.0872$ \\
Per Epoch (s) & $29.1660$ & $95.2178$ & $167.5718$ & $46.4317$ & $44.0621$ \\
\bottomrule
\end{tabular}
\end{table}

\begin{table}[t]
\small
\setlength{\tabcolsep}{1.mm}
\centering
\caption{Experimental results of the comparison between energy discrepancy and contrastive divergence with varying MCMC steps.}
\label{tab:image_gen_cd-n}
\begin{tabular}{l|ccccccccc}
\toprule 
Dataset $\backslash$ Method & CD-1 & CD-3 & CD-5 & CD-7 & CD-10 & \EDB & \EDG \\
\midrule 
Static MNIST & $182.53$ & $130.94$ & $102.70$ & $98.07$ & $\bf 88.13$ & $96.11$ & $90.61$\\
Dynamic MNIST & $157.14$ & $130.56$ & $97.50$ & $91.00$ & $\bf 84.16$ & $97.12$ & $90.19$\\
Omniglot & nan. & $161.96$ & $142.91$ & $149.68$ & $146.11$ & $97.57$ & $\bf 93.94$ \\
\bottomrule
\end{tabular}
\end{table}

\textbf{The Efficacy of the Number of Negative Samples.}
In all experiments, we choose the number of negative samples as $M=32$ irrespective of the dimension of the problem, to maximise computational efficiency within the constraints of our GPU capacity.
\begin{wraptable}{r}{0.5\linewidth}
    \small
    \centering
    \caption{Discrete image modelling results of \EDG on the static MNIST dataset with different $M$ and $w=1$.}
    \label{tab:image_diff_M}
    \vspace{-2mm}
    \begin{tabular}{l|cccc}
    \toprule 
     & $M=4$ & $M=8$ & $M=16$ & $M=32$ \\
    \midrule 
    NLL & $90.13$ & $90.37$ & $89.14$ & $90.61$ \\
    \bottomrule
    \end{tabular}
\vspace{-2mm}
\end{wraptable}
To investigate the impact of the number of negative samples on performance, we conduct experiments by training energy-based models on the static MNIST dataset with \EDG for different values of $M$.
As detailed in \cref{tab:image_diff_M}, our results maintain comparable quality even as the number of negative samples is decreased. Notably, our approach offers greater parallelisation potential compared to the sequentially computed MCMC of contrastive divergence.

\subsection{Training Ising Models} \label{appendix-sec-training-ising-models}

\begin{wrapfigure}{r}{0.50\linewidth}
\vspace{-4mm}
\centering
\includegraphics[width=.16\textwidth]{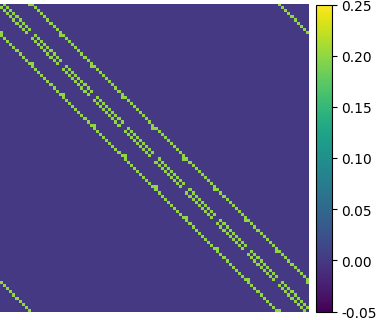}
\includegraphics[width=.16\textwidth]{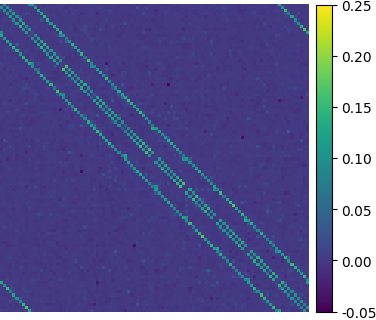}
\includegraphics[width=.16\textwidth]{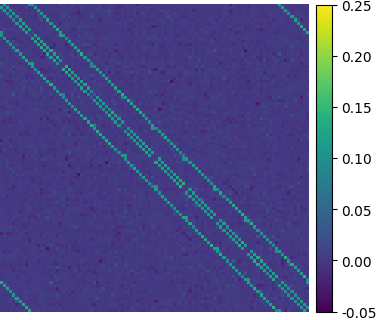}
\vspace{-2mm}
\caption{Results on learning Ising models. Left to right: ground truth, \EDB, \EDG.}
\label{fig:ising_learned_matrix}
\end{wrapfigure}
\textbf{Task Descriptions.} We further evaluate our methods for training the lattice Ising model, which has the form of
\begin{align}
    p(\x) \propto \exp (\x^T J \x), \  \x \in \{-1,1\}^D, \nonumber
\end{align}
where $J=\sigma A_D$ with $\sigma \in \mathbb{R}$ and $A_D$ being the adjacency matrix of a $D\times D$ grid.
Following \cite{grathwohl2021oops,zhang2022langevin,zhang2022generative}, we generate training data through Gibbs sampling and use the generated data to fit a symmetric matrix $J$ via energy discrepancy.
Note that the training algorithms do not have access to the data-generating matrix $J$, only to the collection of samples.

\textbf{Experimental Details.} As in \cite{grathwohl2021oops,zhang2022generative,zhang2022langevin}, we train a learnable connectivity matrix $J_\phi$ to estimate the true matrix $J$ in the Ising model. To generate the training data, we simulate Gibbs sampling with $1,000,000$ steps for each instance to construct a dataset of $2,000$ samples. For energy discrepancy, we choose $w=1,M=32$ for all variants, $\epsilon=0.1$ in Bernoulli perturbations. The parameter $J_\phi$ is learned by the Adam \citep{kingma2014adam} optimiser with a learning rate of $0.0001$ and a batch size of $256$. 
Following \cite{zhang2022generative}, all models are trained with an $l_1$ regularisation with a coefficient in $\{100, 50, 10, 5, 1, 0.1, 0.01\}$ to encourage sparsity. The other setting is basically the same as Section F.2 in \cite{grathwohl2021oops}. We report the best result for each setting using the same hyperparameter searching protocol for all methods.

\textbf{Qualitative Results.}
In \cref{fig:ising_learned_matrix}, we consider $D=10\times 10$ grids with $\sigma =0.2$ and illustrate the learned matrix $J$ using a heatmap. It can be seen that the variants of energy discrepancy can identify the pattern of the ground truth, confirming the effectiveness of our methods.

\begin{wraptable}{r}{0.6\linewidth}
\vspace{1.5mm}
    \small
    \centering
    \caption{Mean negative log-RMSE (higher is better) between the learned connectivity matrix $J_\phi$ and the true matrix $J$ for different values of $D$ and $\sigma$. The results of baselines are directly taken from \cite{zhang2022generative}.} 
    \label{tab:ising_results}
    \vspace{-2mm}
    \resizebox{\linewidth}{!}{
   \begin{tabular}{lccccccc}
        \toprule
        \multirow{2}{*}{Method $\backslash$ $\sigma$} & \multicolumn{5}{c}{$D=10^2$} & \multicolumn{2}{c}{$D=9^2$} \\
        \cmidrule(lr){2-6}\cmidrule(lr){7-8}
          & $0.1$ & $0.2$ & $0.3$ & $0.4$ & $0.5$ & $-0.1$ & $-0.2$  \\ 
         \midrule
        Gibbs & $4.8$ & $4.7$ & $\bf 3.4$ & $\bf 2.6$ & $\bf 2.3$ & $4.8$ & $4.7$ \\
        GWG & $4.8$ & $4.7$ & $\bf 3.4$ & $\bf 2.6$ & $\bf 2.3$ & $4.8$ & $4.7$ \\
        EB-GFN & $\bf 6.1$ & $\bf 5.1$ & $3.3$ & $\bf 2.6$ & $\bf 2.3$  & $\bf 5.7$ & $\bf 5.1$ \\
        \midrule
        \EDB & $5.1$ & $4.0$ & $2.9$ & $2.5$ &  $\bf 2.3$ & $5.1$ & $4.3$ \\
        \EDG & $4.6$ & $4.0$ & $3.1$ & $\bf 2.6$ & $\bf 2.3$ & $4.5$ & $4.0$ \\
        \bottomrule
    \end{tabular}
    }
\vspace{-2mm}
\end{wraptable}
\textbf{Quantitative Results.}
In the quantitative comparison to the baselines,
we consider $D=10\times 10$ grids with $\sigma = 0.1, 0.2, \dots, 0.5$ and $D=9\times 9$ grids with $\sigma=-0.1, -0.2$. The methods are evaluated by computing the negative log-RMSE between the estimated $J_\phi$ and the true matrix $J$. As shown in \cref{tab:ising_results}, our methods demonstrate comparable results to the baselines and, in certain settings, even outperform Gibbs and GWG, indicating that energy discrepancy is able to discover the underlying structure within the data.

\subsection{Graph Generation} \label{appendix-sec-graph-generation}

\textbf{Task Descriptions.}
The efficacy of our methods can be further demonstrated by producing high-quality graph samples. Following the setting in \cite{you2018graphrnn}, our model is evaluated on the Ego-small dataset, which comprises one-hop ego graphs extracted from the Citeseer network  \citep{sen2008collective}.
We consider the following baselines\footnote{There is insufficient information to reproduce EBM (GwG) and RMwGGIS precisely from \cite{liu2023RMwGGIS}. We reran these two baselines with controlled hyperparameters for a fair comparison, while other baseline results were taken from their original papers.} in graph generation, including GraphVAE \citep{simonovsky2018graphvae}, DeepGMG \citep{li2018learning}, GraphRNN \citep{you2018graphrnn}, GNF \citep{liu2019graph}, GrappAF \citep{shi2020graphaf}, GraphDF \citep{luo2021graphdf}, EDP-GNN \citep{niu2020permutation}, RMwGGIS \citep{liu2023RMwGGIS}, and contrastive divergence with GWG sampler \citep{grathwohl2021oops}.

\begin{wraptable}{R}{0.45\linewidth}
	\caption{Graph generation results in terms of MMD. \textit{Avg.} denotes the average over three MMD results.}
	\label{tab:gg-quantitative}
        \vspace{-2mm}
	\centering
	\resizebox{1.0\linewidth}{!}{
	\begin{tabular}{lccc|c}
		\toprule
		\textbf{Method} & \textit{Degree} &\textit{Cluster} &\textit{Orbit} &\textit{Avg.}  \\
		\midrule
	    GraphVAE & $0.130$ & $0.170$ & $0.050$ & $0.117$ \\
	    DeepGMG & $0.040$ & $0.100$ & $0.020$ & $0.053$ \\
	    GraphRNN & $0.090$ & $0.220$ & $0.003$ & $0.104$ \\
	    GNF & ${0.030}$ & $0.100$ & $0.001$ & $0.044$ \\
	    GraphAF & ${0.030}$ & $0.110$ & ${0.001}$ & $0.047$ \\
	    GraphDF & $0.040$ & $0.130$ & $0.010$ & $0.060$ \\
	    EDP-GNN & $0.052$ & $0.093$ & $0.007$ & $0.050$ \\
	   EBM (GWG) & $0.095$ & $0.061$ & $0.032$ & $0.063$ \\
	    RMwGGIS & $0.066$ & $0.042$ & $0.036$ & $0.048$ \\
	    \midrule
            \EDB & $0.063$ & $0.054$ & $0.014$ & $0.044$ \\
            \EDG & $0.036$ & $0.050$ & $0.019$ & $\mathbf{0.035}$ \\
		\bottomrule
	\end{tabular}
	}
	\vspace{-3mm}
\end{wraptable}
\textbf{Experimental Details.}
Following the setup in \cite{you2018graphrnn}, we split the Ego-small dataset, allocating $80\%$ for training and the remaining $20\%$ for testing. To provide better insight into this task, we illustrate a subset of training data in \cref{fig:gg-training-data}. Notably, these training data examples closely resemble realistic one-hop ego graphs.

For a fair comparison, we parametrise the energy function via a 5-layer GCN \citep{kipf2016semi} with the ReLU activation and 16 hidden states for all energy-based approaches. For hyperparameters, we choose $M=32, w=1$ for all variants of energy discrepancy and $\epsilon=0.1$ for the Bernoulli perturbation. Following the configuration in \cite{liu2023RMwGGIS}, we apply the advanced version of RMwGGIS with the number of
samples $s=50$ \cite[Equation 11]{liu2023RMwGGIS}. 
Regarding the EBM (GWG) baseline, we train it using persistent contrastive divergence with a buffer size of $200$ samples and the MCMC steps being $50$. 
To train the models, we use the Adam optimiser with a learning rate of $0.0001$ and a batch size of $200$.
After training, we generate new graphs by first sampling $N$, which is the number
of nodes to be generated, from the empirical distribution of the number of nodes in the training dataset, and then applying the GWG sampler \citep{grathwohl2021oops} with 50 MCMC steps from a randomly initialised Bernoulli noise. 
To assess the quality of these samples, we employ the MMD metric, evaluating it across three graph statistics, i.e., degrees, clustering coefficients, and orbit counts.
Following the evaluation scheme in \cite{liu2019graph}, We trained 5 separate models of each type and performed 3 trials per model, then averaged the result over 15 runs.

\textbf{Qualitative Results.}
We provide a visualisation of generated graphs from variants of our methods in \cref{fig:gg-samples-edb,fig:gg-samples-edg}. Notably, the majority of these generated graphs resemble one-hop ego graphs, illustrating their adherence to the graph characteristics in the training data.

\textbf{Quantitative Results.}
In \cref{tab:gg-quantitative}, we compare our methods to various baselines. It can be seen that our methods outperform most baselines in terms of the average of the three MMD metrics, indicating the faithful energy landscapes learned by the energy discrepancy approaches.

\begin{figure}[!t]
    \centering
    \begin{subfigure}{0.32\linewidth}
        \centering
        \includegraphics[width=1.8in]{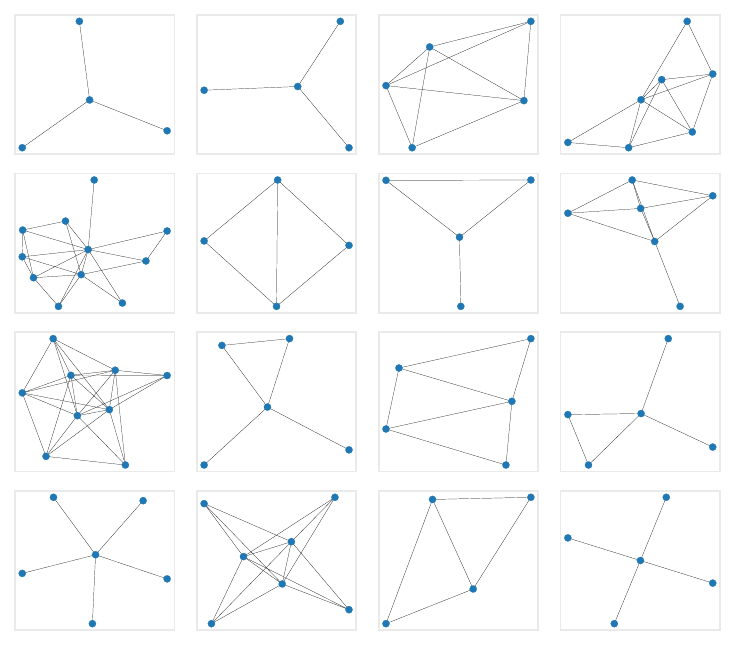}
        \vspace{-5mm}
        \caption{Training Data}
        \label{fig:gg-training-data}
    \end{subfigure}
    \hfill
    \begin{subfigure}{0.32\linewidth}
        \centering
        \includegraphics[width=1.8in]{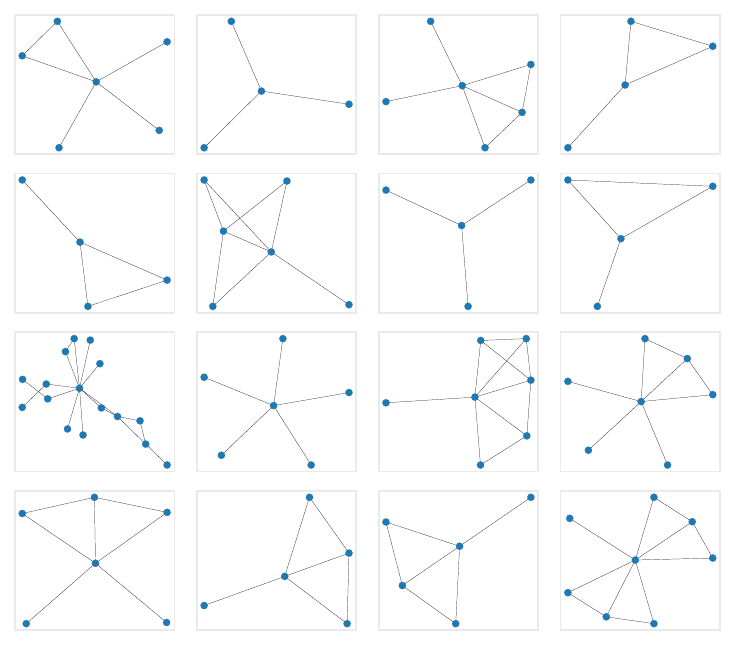}
        \vspace{-5mm}
        \caption{\EDB}
        \label{fig:gg-samples-edb}
    \end{subfigure}
    \hfill
    \begin{subfigure}{0.32\linewidth}
        \centering
        \includegraphics[width=1.8in]{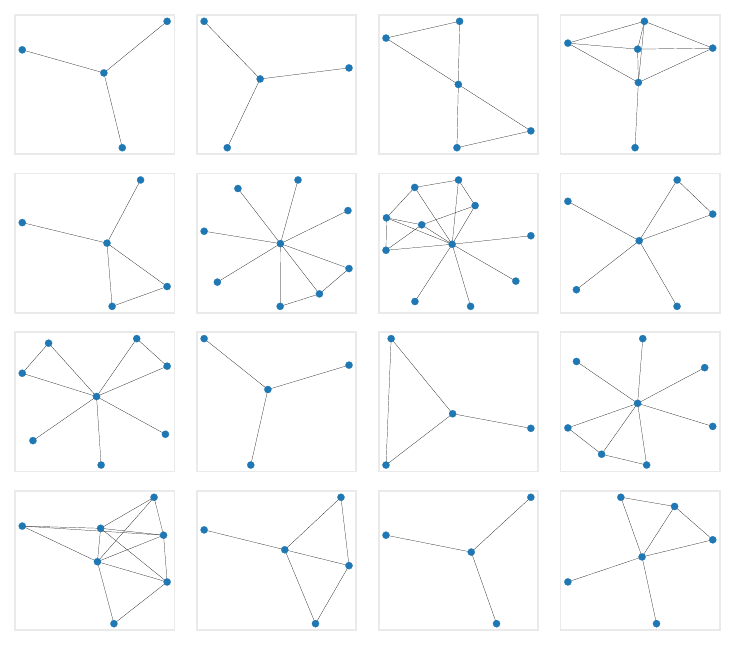}
        \vspace{-5mm}
        \caption{\EDG}
        \label{fig:gg-samples-edg}
    \end{subfigure}
    \caption{Visualisation of the training data and samples drawn from the energy-based models learned by the variants of our approaches on the Ego-small dataset.}
    \vspace{-3mm}
\end{figure}
\newpage
\section{Naming Conventions and Parameters of Introduced Methods}
This table summarises the naming conventions and available tuning parameters for all introduced methods. The structured perturbation TabED-Str uses different perturbations depending on the state space structure: On unstructured data, the uniform perturbation with tuning hyper-parameter $t_{\mathrm{cat}}$ is used, while on ordinally and cyclically structured data the ordinal perturbations and cyclical perturbations are used, respectively, with tuning parameter $t_{\mathrm{base}}$.
\begin{table*}[h!]
\small
    \caption{Overview of all introduced energy discrepancy methods}
    \centering
    \hspace{-2mm}
\begin{tabular}{l|c c c}
\toprule
    Name & Space (Discrete component) & Perturbation (Discrete component) & Tuning Parameter \\
\midrule
    \EDB & $\{0, 1\}^d$ & $\prod_{k=1}^d \mathrm{Bern}(\varepsilon)$ & $\varepsilon = 0.5(1-e^{-2t})$ \\
    \EDG & $\{0, 1\}^d$ & $\sum_{k= 1}^d \frac{1}{d}\delta_{\vert y_k - x_k\vert = 1}$ & None \\
    TabED-Uni & $\otimes_{k=1}^d\{1, \dots, S_k\}$ & $\prod_{k=1}^d \exp(tR^{\mathrm{unif}})_{y_k x_k}$ (\cref{equ:uniform_perturbation_estimation}) & $t >0$ \\
    TabED-Grid & $\otimes_{k=1}^d\{1, \dots, S_k\}$ & $\sum_{k=1}^d \frac{1}{d} \delta(y_k, \Box)\delta(\mathbf y_{\neg k}, \mathbf x_{\neg k})$ & None \\
    TabED-Cyc & $\otimes_{k=1}^d\{1, \dots, S_k\}$ & $\prod_{k=1}^d \exp(t_kR^{\mathrm{cyc}})_{y_k x_k}$ (\cref{prop:SolutionHeatEquation}) & $t_k = S_k^2 t_{\mathrm{base}}$ \\
    TabED-Ord & $\otimes_{k=1}^d\{1, \dots, S_k\}$ & $\prod_{k=1}^d \exp(t_kR^{\mathrm{ord}})_{y_k x_k}$ (\cref{prop:SolutionHeatEquation}) & $t_k = S_k^2 t_{\mathrm{base}}$\\
    TabED-Str & $\otimes_{k=1}^d\{1, \dots, S_k\}$ & $\prod_{k=1}^d \exp(t_kR^k)_{y_k x_k}$ (Mixed) & Mixed\\
    \bottomrule
\end{tabular}
\end{table*}


\end{document}